\documentclass{article}

\usepackage[utf8]{inputenc} % allow utf-8 input
\usepackage[T1]{fontenc}    % use 8-bit T1 fonts

\usepackage{amsmath}
\usepackage{amssymb}
\usepackage{amsfonts}
\usepackage{accents}
\usepackage{dsfont}
\usepackage{appendix}
\usepackage[mathscr]{euscript}

\usepackage{enumitem}
\usepackage[table,xcdraw]{xcolor}
\usepackage[colorlinks=true,linkcolor=blue]{hyperref}

\usepackage{url}            % simple URL typesetting
\usepackage{nicefrac}       % compact symbols for 1/2, etc.
\usepackage{microtype}      % microtypography

\usepackage{amsthm}

\usepackage[]{algorithm}
\usepackage{algorithmicx}
\usepackage[noend]{algpseudocode}

\usepackage{graphicx}
\usepackage{caption}
\usepackage[skip=0cm,list=true,labelfont=it]{subcaption}

\usepackage{booktabs} % for professional tables

\usepackage[final]{neurips_2022}

\usepackage{adjustbox}
\usepackage{bbm}

\newcolumntype{x}[1]{>{\centering\let\newline\\\arraybackslash\hspace{0pt}}m{#1}}
\definecolor{DarkBlue}{rgb}{0.1,0.1,0.5}
\definecolor{DarkGreen}{rgb}{0.1,0.5,0.1}
\hypersetup{
	colorlinks   = true,
	linkcolor    = DarkBlue, % color of internal links
	urlcolor     = DarkBlue, % color of external links
	citecolor    = DarkGreen % color of links to bibliography
}

\newtheoremstyle{thmstyle}
{0.5em} % Space above
{0.15em} % Space below
{} % Body font
{} % Indent amount
{\bfseries} % Theorem head font
{.} % Punctuation after theorem head
{.5em} % Space after theorem head
{} % Theorem head spec (can be left empty, meaning `normal')

\theoremstyle{thmstyle} 
\newtheorem{thm}{Theorem}

\newtheorem{propn}{Proposition}[section]

\newtheorem*{claim*}{Claim}

\newtheorem{appendixpropn}{Proposition}[subsection]

\theoremstyle{definition}

\theoremstyle{remark}
\newtheorem{rem}{Remark}

\newcommand{\comment}[1]{}

\newcommand{\E}{\mathbb{E}}
\newcommand{\F}{\mathcal{F}}
\newcommand{\h}{\mathring{h}}
\renewcommand{\H}{\mathcal{H}}

\newcommand{\M}{\mathcal{M}}
\newcommand{\Mdot}{\mathring{\mathcal{M}}}

\renewcommand{\P}{\mathcal{P}}

\newcommand{\prob}{\mathbb{P}}

\newcommand{\R}{\mathbb{R}}
\renewcommand{\S}{\mathcal{S}}
\newcommand{\T}{\tau}
\newcommand{\U}{\mathcal{U}}
\newcommand{\vepsilondot}{\mathring{\varepsilon}}
\newcommand{\SIM}{\texttt{SIM-Algorithm}}

\newcommand{\0}{\underline{\mathbf{0}}}

\newcommand{\sltoinf}{\sum\limits_{i=0}^\infty}

\newcommand{\slktoinf}{\sum\limits_{k=0}^\infty}

\newcommand{\slsins}{\sum\limits_{s\in\S}}

\newcommand{\sltotau}{\sum\limits_{i=1}^\tau}

\DeclareMathOperator*{\argmin}{arg\,min}

\DeclareMathOperator*{\expectation}{\,\mathbb{E}\,}

\newcommand{\tr}{\top}

\title{Scale Invariant Monte Carlo under Linear Function Approximation with Curvature based Step-size}

\author{%
	Rahul Madhavan,\\
	Department of Computer Science and Automation\\
	Indian Institute of Science\\
	Bangalore, India\\
	\texttt{mrahul@iisc.ac.in} \\
	\And
	Hemanta Makwana,\\
	Department of Computer Science and Automation\\
	Indian Institute of Science\\
	Bangalore, India\\
	\texttt{hemantam@alum.iisc.ac.in} \\
}

\begin{document}

	\maketitle
	
	\begin{abstract}
		We study the feature-scaled version of the Monte Carlo algorithm with linear function approximation. This algorithm converges to a scale-invariant solution, which is not unduly affected by states having feature vectors with large norms. 
		The usual versions of the MCMC algorithm, obtained by minimizing the least-squares criterion, do not produce solutions that give equal importance to all states irrespective of feature-vector norm -- a requirement that may be critical in many reinforcement learning contexts.
		To speed up convergence in our algorithm, we introduce an adaptive step-size based on the curvature of the iterate convergence path -- a novelty that may be useful in more general optimization contexts as well. A key contribution of this paper is to prove convergence, in the presence of adaptive curvature based step-size and heavy-ball momentum. We provide rigorous theoretical guarantees and use simulations to demonstrate the efficacy of our ideas.
	\end{abstract}
	
	\vspace{-0.05in}	
	\section{INTRODUCTION}
	\label{section:intro to paper}
	
	\vspace{-0.05in}
	Feature scaling and data normalization is a common practice in machine learning and has been shown to be effective in a variety of areas such as deep learning \citep{BishopPage298,SolaInputNormalization}, nearest neighbour classifiers \citep{LiZhangLiFeatureScalinginNearestNeighbour,SinghNormalizationClassifiers}, SVMs \citep{FeatureScalingSVMs}, PCA \citep{ISLCasellaBerger} and data mining \citep{DataMiningHan}. Their  main utility is when the norm of the input vector is not a true reflection of its importance  \cite{BishopPage298}. Normalization is also known to often help increase the speed of learning  \cite{BaHinton2016} as well as reduce the dependence on outliers \citep{BenGal2005Outlier,Botchkarev_2019}.
	
	\vspace{-0.05in}
	Finding the optimal policy in Markov Decision Processes (MDPs) remains the central goal of reinforcement learning. 
	In the context of optimal control, the value-iteration \citep{Bellman1957} and policy iteration algorithms \citep{howard1960dynamic} have remained the cornerstones of dynamic programming (DP) methods to solve this problem. When one doesn't know the model (transition probabilities) in the MDP explicitly, algorithms like Monte Carlo, TD(0) Learning, TD($\lambda$) Learning and Q-Learning, or their variations are often used \citep{suttonBarto}. 
	
	For a large class of problems, the state space becomes large enough that explicitly maintaining the values associated with each state becomes infeasible \cite{csaba}. In such cases, one uses approximation techniques to model the values associated with states. Two such approximation techniques that are often used are linear function approximation and neural network approximation.
	
	To rigorously prove that the above methods work as expected, one needs to provide theoretical guarantees of their convergence. In the tabular setting (with no function approximation), several theoretical results provide such guarantees (for example, see \citet{dayan1992convergence,DayanTDlambdaConvergence,tsitsiklis2002convergence}). In the context of linear function approximation, stochastic approximation techniques and ODE methods such as those listed in \citet{convergenceStochasticApproxLennart,borkar2000ode,Borkar,Kushner1997} are often used to provide such guarantees. Convergence guarantees under the linear function approximation regime have been explored in works by \citet{tsitsiklis1996analysis,korda2015td,bertsekas2004improved, konda1999actor,perkins2002convergent,bertsekas2011approximate}.

	In the linear function approximation setup, the value assigned to any state is approximated by a linear function of the feature vector associated with the state. For instance, in an $m$-state MDP, if the feature vector associated with state $i,\enspace i\in [m]$\footnote{We use $[m]$ to indicate $\{1,\dots,m\}$} is $\phi_i$, then for some weight vector $w$, the value $V_i$ associated with the state $i$ is approximated by $\phi_i^\tr w$. 
	If $\Phi$ is the matrix with rows as the feature vectors, and V is the vector of values associated with the states, i.e $\Phi = \left\{\phi_i\right\}_{i\in[m]}$, and $V = \left\{V_i\right\}_{i\in[m]}$, then we are approximating $V$ by $\Phi^\tr w$.
	Using the least-squares criterion to find $w$ leads us to $w = \argmin_{w'} ||\Phi^\tr w' - V ||_2^2$. More generally, if we assign weights $d_i$ to each state $i$ such that $\sum_{i\in[m]} d_i=1$, then, the least squares (LS) criterion gives a weight $w = \argmin_{w'} \sum_{i\in[m]} d_i (\phi_i^\tr w' - V_i)^2$. The major reinforcement learning algorithms using linear function approximation (listed previously) obtain weight vectors that conform to this criterion.
	
	\textbf{Illustrative example for issues with LS:} The least squares method provides solutions which are more skewed towards feature vectors with larger norm. \label{example: toy example} We illustrate this with a toy-example as follows. Consider a two-state system with features $\phi_1,\phi_2\in\R$. Say the values associated with these two states are $V_1,V_2$. Let $\phi_1=1,\phi_2=2$ and $V_1=2,V_2=1$. 
	Then we want some $w\in\R$ such that $w\simeq2$ and $2w\simeq1$. One may expect the answer to be the mean of $2$ and $\frac{1}{2}$, i.e.  $w=\frac{5}{4}$, but the least squares solution for this system is $w=\frac{4}{5}$. The least squares solution is dominated by the second feature vector, viz $\phi_2 = 2$, thus gives a solution that approximates the second linear equation better. The issue highlighted by this example is exacerbated when states that have features that are outliers.
	
	%$\phi_i\cdot w$ approximates $V_i$ for $i\in\{1,2\}$. One might expect that such a $w$ provides equal error on both the states. This would give us $w=\frac{2}{3}$ wherein $|\phi_1 w - V_1| = |\phi_2 w - V_2| = \frac{1}{3}$. One can calculate that the least squares criterion gives a solution $w=\frac{3}{5}$.
	%Consider a toy example of an overdetermined $2\times 1$ system: $2w=1$ and $w=2$. One may expect the answer to be the mean of $\frac{1}{2}$ and 2, i.e.  $w^*=\frac{5}{4}$, but the least squares solution for this system is $w^*=\frac{4}{5}$. The least squares solution is dominated by the first feature vector, viz $2$, thus gives a solution closer to the first linear equation. 
	
	To address this issue, in this paper, we propose a solution calculated as per the alternative criterion: 
	$w = \argmin_{w'}  \sum_{i = 1}^m d_i (\phi_i^\tr w' - V_i)^2/\|\phi_i\|_2^2$
	which is the minimizer of the weighted sum of squares of distances from $w'$ to the hyperplanes $\phi_i^\tr w = V_i,$ $i \in[m]$. This criterion has the following two advantages over least squares. Firstly, the solution under our criterion is scale invariant, i.e., irrespective of the norm (scale) of the feature vectors, the solution gives importance to states proportional to the chosen $d_i$ values.
	Secondly, the solution is more robust to outlier rows. Unlike in the least-squares solution, large $\phi_i$ values which may be outliers will not unduly affect the solution.
	Further, the solution remains unchanged even if the individual equations are re-scaled.
	
	While scaling of feature vectors is often used in practice in a variety of machine learning as well as reinforcement learning contexts \citep{ioffe2015batch,santurkar2018does,huang2020normalization,bhatt2019crossnorm}, the current work contributes to the theory relating to  feature-scaling in the context of RL algorithms. We provide convergence guarantees in the presence of momentum as well as an adaptive step size method.
	
	We now present related work that the current paper builds upon. These broadly touch upon three aspects -- adaptive step size, convergence under momentum and feature-normalization and scaling.
	
	\vspace{-0.05in}
	\subsection{Additional Related Work}
	
	\vspace{-0.05in}	
	%The classical stochastic approximation scheme presented by \citet{robbins1951stochastic}, presents a deterministic step size sequence choice that several works have tried to improve upon. 
	Adaptive step sizes have been explored classically by  \citet{adaptiveStepSizeRandomSearch1968,ang2001new,kushner1994analysis} amongst others. In the context of reinforcement learning, adaptive step sizes have been explored in the context of policy gradient \citep{pirotta2013adaptive}, and temporal difference learning \citep{dabney2012adaptive}. 
	
	In optimization literature, several stochastic gradient descent (sgd) based algorithms use some form of adaptive step size \cite{ruder2016overview}. Many like Adagrad \citep{duchi2011adaptive} and Adadelta \citep{zeiler2012adadelta} modify the step size. Others like Adam \citep{AdamKingma} also add additional momentum terms to speed up convergence. A recent work also adapts the Polyak step-sizes to be stochastically updated \citep{loizou2021stochastic}.
	Convergence of some of these methods in the presence of momentum have been studied recently in works by \citet{reddiAdamConvergence,defossez2020convergence,mai2020convergence,chen2018on,yang2016unified}
	
	In other threads of work, normalization and feature scaling have been studied to good effect in the non-convex landscape of neural networks. For instance, layer normalization \citep{ba2016layer} and batch normalization \citep{ioffe2015batch} have been used to ``normalize'' activations in intermediate layers of neural networks. Group normalization \citep{wu2018group}, self-normalization \citep{klambauer2017self}, weight normalization \citep{salimans2016weight} and other variants have also been considered. Some works propose that these techniques make the optimization landscape smoother \citep{santurkar2019howbatchnormhelps}, and other works propose that they help reduce covariate shift \citep{ioffe2015batch}. We note that these normalization techniques rescale the inputs based on statistics per set of inputs, rather than a re-scaling of each input to have norm 1. 
	
	Normalization and feature scaling are less often used in linear settings
	-- possibly because the drawbacks of using inputs that are not feature-scaled are not apparent. As highlighted in our \hyperref[example: toy example]{toy example}, using features without scaling in methods like minimization of least squares, can lead to solutions that are more skewed towards data where the feature-norms are higher. 
	
	Our work focuses on this problem of feature scaling in linear settings where we provide convergence guarantees in the presence of adaptive step size and momentum.
	
	\vspace{-0.05in}
	\subsection{Our Contributions}
	\label{Section: Our Contributions}
	
	\vspace{-0.05in}
	We formulate and study the convergence of scale-invariant RL algorithms with linear function approximation in the presence of momentum and adaptive step size. Our algorithm uses a variant of the stochastic Kaczmarz method \citep{strohmerVershynin} to seek a scale-invariant solution. Note that the original method solves overdetermined $\Phi w=V$ systems that are consistent, and requires access to exact value ($V$) estimates. We provide a convergence guarantee even with only noisy samples of the value function. This is crucial in RL applications, where we get access to some noisy estimate of the value either by a one step temporal-difference (TD) or by summing rewards (Monte Carlo).
	
	In RL systems, every state might be equally important irrespective of the feature vector norms. Our algorithms converge to a solution that satisfies this property of not being unduly influenced by outliers, or states with high feature-vector norms. Hence, we call our algorithm scale-invariant --- as the scale of the input features does not matter to the output solution.
	
	We now outline the basic (linear) framework under which our Algorithms operate.
	
	\vspace{-0.05in}
	\subsubsection{The Update Rule}
	\label{section: The Update Rule}
	
	\vspace{-0.05in}	
	Consider any overdetermined linear system  $\Phi w = V,$ consisting of m rows of the form $\phi_i^\tr w = V_i,\enspace \phi_i,w\in\R^n$.
	Let $D$ be a diagonal weight matrix with entries $d_1, \ldots, d_m.$ 
	If we wish to solve 
	\begin{equation}
		\label{eqn: wstar in TP}
		w^* = \min_w \sum_{i = 1}^m d_i (\phi_i w - V_i)^2/\|\phi_i\|^2
	\end{equation} Then the stochastic update (with say $\T$ samples) takes the form
	\begin{equation}
		w_{k + 1} = w_k - \alpha_k \frac{1}{\T}\sum\limits_{i=1}^\T \frac{\phi_i^\tr w_k - V_i}{\|\phi_i\|^2} \phi_i
	\end{equation}
	where the rows are sampled with probability $d_i$ and $\alpha_k$ is some step-size sequence. 
	%This update rule leads to  \hyperref[algo:TPAlgorithm]{Algorithm \ref{algo:TPAlgorithm}}.
	Since each expression of the form $(\phi_i^\tr w_k - V_i)\phi_i/(\|\phi_i\|^2)$ is a projection from $w_k$ onto the hyperplane $\phi_i^\tr w = V_i$, we call the update map from $w_k$ to $w_{k+1}$ for all iterations $k$ as Total Projections (TP) map. In general, such a map changes per iteration as a different rows of the form $\phi_i^\tr w = V_i$ are chosen. Depending on such a choice, the map at step $k$ may be called $TP_k:\R^n \to \R^n$. 
	In other words, $w_{k+1}=TP_k(w_k)$. This is a Kaczmarz based algorithm \citep{kaczmarz_1937} which converges to a consistent solution in the presence of no noise. 
	
	%Since each expression of the form $(\phi_i^\tr w_k - V_i)\phi_i/(\|\phi_i\|^2)$ is a projection from $w_k$ onto the hyperplane $\phi_i^\tr w = V_i$, 
	
	For our full update rule, we need to add a momentum and our choice of step size. For the momentum part we use heavyball momentum with constant $\beta$ (reasons in \hyperref[section: Choice of Momentum in Total Projections]{Section \ref{section: Choice of Momentum in Total Projections}}). For our step size, we use an osculating circle based step choice, which we call the curvature step (details in \hyperref[section:StepSizeSeq]{Section \ref{section:StepSizeSeq}}). We provide evidence that the step size works in \hyperref[section: adaptive step size]{the section \ref{section: adaptive step size}}. With these in place, we now describe our full update rule. 
	
	Let $TP_k(w)=\frac{1}{\T}\sum_{i = 1}^\T (\phi_i w - V_i)\phi_i/\|\phi_i\|^2$,
	and $\Delta TP_k(w_k)$ = $TP_k([w_k - TP_k(w_k)]) - TP_k(w_k)$. Here the stochastic gradient update on $w_k$ with respect to our error term is given by $TP_k(w_k)$ and $\Delta TP_k(w_k)$ indicates the change in gradient. Our update rule is then given by:
	
	\vspace{-0.05in}
	\begin{equation}\label{eqn:full TP update equation}
		w_{k+1} = w_k - \eta_k\dfrac{||TP_k(w_k)||}{||\Delta TP_k(w_k)||} TP_k(w_k) + \beta(w_{k}- w_{k-1})
	\end{equation}
	where $\beta \in (0,1)$ and $\eta_k = 1/k^p;\quad p\in(0.5,1]$. We now provide an intuition for each of the terms.
	
	The second term in \hyperref[eqn:full TP update equation]{Equation \ref{eqn:full TP update equation}} indicates the gradient update. Note that $\eta_k$ is a decreasing step-size sequence. Typically, one might use some sequence $\{\eta_k\}$ such that $\sum_{k=1}^\infty \eta_k = \infty$ and $\sum_{k=1}^\infty \eta_k^2 < \infty$. Such a requirement is satisfied by $\eta_k = 1/k^p$ where $p\in(0.5,1]$ \citep{robbins1951stochastic,blum54Convergence}. Recall that $TP_k$ is a map that gives the projections over the sampled hyperplanes $\phi_i w = V_i$. Therefore $TP_k(w_k)$ is an update in the direction of the required (total) projection from $w_k$ onto the hyperplanes chosen, i.e. a gradient descent update from $w_k$ towards $w^*$ in our chosen error metric.
	
	The updates to $w_k$ at discrete time steps may be assumed to be a noisy discretization to a continuous curve $\omega(t)$ at some time $t$ such that $\omega(t) = w_k$. Then $\omega'(t)$ -- the tangent to the curve $\omega(t)$ -- may be approximated by the update $TP_k(w_k)$. 
	The unit tangent to the curve $\omega(t)$ is given by $TP_k(w_k)/||TP_k(w_k)||$. Further, $\omega''(t)$ is approximated by the update $\Delta TP_k(w_k)$.
	But the radius of curvature $\kappa = ||dT/d \omega||$ where $T(t)$ is the unit tangent at time $t$, and $\omega(t)$ is the parameterized curve \citep{kuhnel2015differential}. 
	
	Then by the Chain Rule, $\kappa = ||dT/dt||/||d\omega /dt|| = ||dT/dt||/||\omega'(t)|| = ||\omega''(t)||/||\omega'(t)||^2$ \citep{Tapp2016}.
	Then we find the approximation $\kappa \simeq ||\Delta TP_k(w_k)||/|| TP_k(w_k)||^2$ for the discrete setting and the radius of curvature $R = 1/\kappa = || TP_k(w_k)||^2/||\Delta TP_k(w_k)||$. Then the update rule becomes $R\cdot$ \texttt{Unit gradient vector} = $R\cdot TP_k(w_k)/|| TP_k(w_k)|| = || TP_k(w_k)||/||\Delta TP_k(w_k)||\cdot  TP_k(w_k)$.
	
	%Let us assume a continuous curve $\omega(t)$ which 
	%Note that for a given parameteric curve $\omega(t)$, parameterized by time, the value $w'(t)$ would indicate the tangent to the curve. This is approximated in the discrete setting by $TP_k(w_k)$ at the point $w_k$.
	
	The third term in \hyperref[eqn:full TP update equation]{Equation \ref{eqn:full TP update equation}} is a heavy-ball momentum term, where we add some constant ($\beta$) times the previous updates. This momentum term, is less useful in the context where we have no noise, but can be useful in the case of noisy updates \citep{gitman2019understanding,sutskever2013importance,polyak1964some}.
	
	In light of the multiple expressions in the update rule given by Equation \ref{eqn:full TP update equation}, 
	%methods outlined in \hyperlink{section: adaptive step size}{Section \ref{section: adaptive step size}} and \hyperlink{section: The Update Rule}{Section \ref{section: The Update Rule}}, 
	showing convergence is not straightforward. We use the theory of stochastic approximation to establish almost sure (a.s.) convergence for the algorithms we propose. This is a key technical contribution of this work.
	
	\vspace{-0.15in}
	\subsubsection{Evidence for Adaptive Step Size}
	\label{section: adaptive step size}
	
	\vspace{-0.05in}
	As outlined in the previous section, the adaptive step size that we choose is derived from the radius of curvature of the continuous curve that approximates our discrete updates in $w_k$. Such a step size sequence, performs quite well in simulations as outlined below. Note that in this simulation, the updates are not noisy. Even allowing for this, the exponential convergence was surprising.
	
	\begin{figure}
		\centering
		\begin{subfigure}[t]{0.49\textwidth}
			%\centerline{\includegraphics[scale=0.64]{images/NoNoise2.png}}
			\centerline{\includegraphics[scale=1]{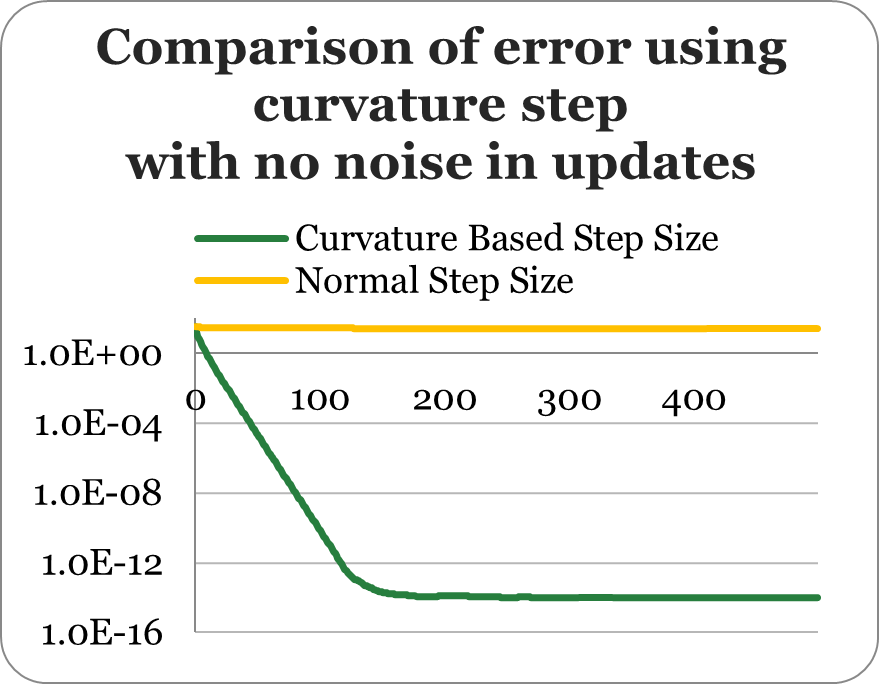}}
			\caption{Error with number of iterations, with curvature step size and normal step size for $m=100$, $n=30$.
			}
			\label{fig:Evidence for Curvature based Step size}
		\end{subfigure}
		\hfill
		\begin{subfigure}[t]{0.49\textwidth}
			%\centerline{\includegraphics[scale=0.64]{images/NoNoiseChangingmn.png}}
			\centerline{\includegraphics[scale=1]{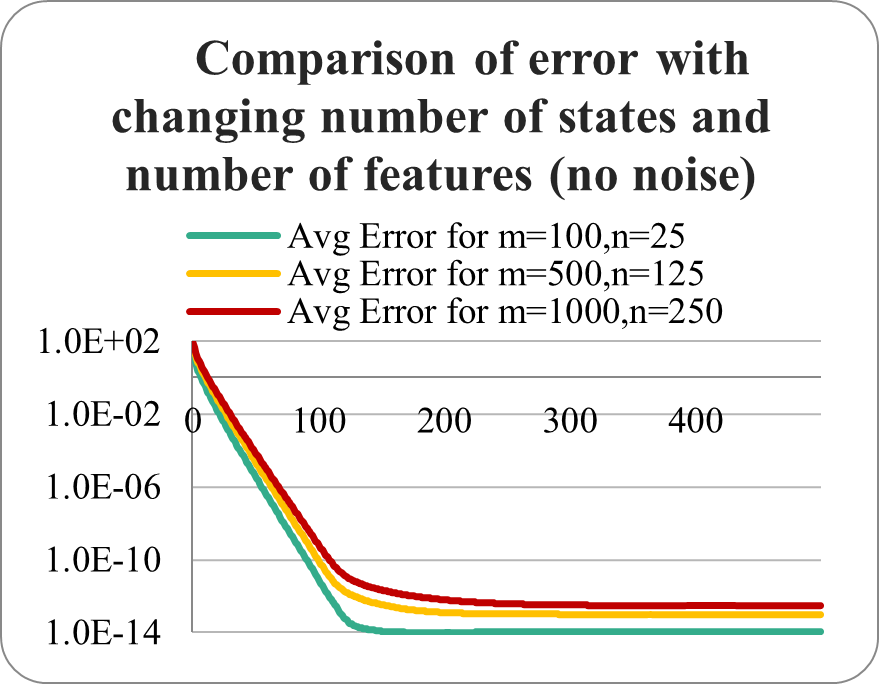}}
			\caption{Error with iterations for (i) $m=100$, $n=30$, (ii) $m=500$, $n=125$, (iii) $m=1000$, $n=250$ 
			}
			\label{fig:Error with changing size of linear system}
		\end{subfigure}
		\caption{Plot for error -- measured as a distance to $w^*$ -- with number of iterations. Note that the convergence rates are much faster using a curvature based step size (figure \ref{fig:Evidence for Curvature based Step size}). The rates of convergence using curvature-step don't change much with changing size of linear system (figure \ref{fig:Error with changing size of linear system}).}
	\end{figure}
	
	\comment{
		\begin{figure}[!htb]
			\vspace{-0.1in}
			\begin{center}
				\centerline{\includegraphics[scale=0.64]{images/NoNoise2.png}}
				\caption{Plot for error -- measured as a distance to $w^*$ -- with number of iterations, with curvature step size and normal step size. The updates are not noisy. We use $m=100$ and $n=30$. Note that the convergence rates are much faster using a curvature based step size.
				}
				\label{fig:Evidence for Curvature based Step size}
			\end{center}
			\vspace{-0.15in}
		\end{figure}
	}
	
	%The adaptive step size sequence chosen by us arose from experimental simulations where this sequence performed quite well. 
	%We propose a novel adaptive step size (in \hyperref[section:StepSizeSeq]{Section \ref{section:StepSizeSeq}}) choice based on the radius of a suitably defined osculating circle, which we obtain from the Frenet-Serret equations. We show our update, in expectation, to be a contraction on the error function, and provide experimental evidence \hyperref[fig:Evidence for Curvature based Step size]{(Figure \ref{fig:Evidence for Curvature based Step size})} for the speedup.

	In \hyperref[fig:Evidence for Curvature based Step size]{figure \ref{fig:Evidence for Curvature based Step size}}, we plot the errors (as measured by distance from the error minimizer $w^*$ for the modified error function as given in equation \ref{eqn: wstar in TP}) with number of iterations for total projections with curvature step algorithm. The number of states $m=100$ and number of features $n=30$. We note the exponential convergence and that the error decreases monotonically on a log-scale. This shows that with the increased curvature-step size, we still have a contraction on the error function.

	\subsubsection{The RL Context}
	Using \hyperref[eqn:full TP update equation]{Equation \ref{eqn:full TP update equation}}, we propose an algorithm Scale Invariant Monte-carlo (\SIM) with curvature step. In the \SIM, the role of $\Phi$ is played by the feature vectors for the states. The value vector for the states $V_i$ is estimated by the First-visit monte carlo where we sum the rewards from state $i$ until termination. Thus, $V_i = \sum_{t=1}^\tau \gamma^{t-1} R_t$ where the state of the Markov Chain at time $t=0$, $s(0) = i$.
	%takes the role of $V$. Note that here, we only have access to noise samples of V, which we estimate through the First-Visit Monte Carlo Algorithm. 
	We note that the sampling of states in the Markov Chain happens as per the stationary distribution of the transition matrix (asymptotically). Thus $\prob\{s(t) = i\enspace \forall t > T_0\} = d_i$ for some large $T_0$.

	\vspace{-0.1in}
	\section{NOTATION AND PRELIMINARIES}
	\label{section: notation and preliminaries}
	
	\vspace{-0.05in}	
	Let us consider an RL setting with state space $\S$, where $|\S| = m$. Let the states be labeled $\{1 \dots m\}$. Consider an Markov Decision Process (MDP) given by $\mathbb{M} = (\S,\mathbb{A},\mathbb{P},R)$ \cite{csaba} and a discount factor $\gamma$. 
	Consider a deterministic stationary policy  $\mu: \S \to \mathbb{A}$. This induces a transition matrix $\P\in  \R^{m\times m}$. $\P$ gives a probability distribution over next states for each given state. The probability of transition from states $s$ to $s'$ ($s,s'\in\S$) is given by $\P_{ss'}$. Given s, the vector of transition probabilities over all $s'\in\S$ is given by $\P_s$. 
	We will assume full mixing and ergodicity. Then let $\pi \in \R^m$ be the stationary distribution associated with $\P$, and $D \in \R^{m\times m}$ be the diagonal matrix associated with vector $\pi$. 
	
	Let $R_{ss'}$ indicates the reward on transition between state $s$ and $s' \enspace (s,s')\in\S$.
	Let $\phi_s \in \R^n$ be the set of features associated with each state and $\Phi$ be the corresponding matrix of all features.
	In the value estimation problem, we want to find the value $V \in \R^m$, under a policy $\mu$, for each state. Then, for each state $s$ we have \cite{csaba} that $V_s = \E [\sum_{t=0}^\T \gamma^t R_{t+1} | S_0 = s ]\label{eqn: value as sum of rewards}$. Under the linear function approximation, we estimate $V$ as $\Phi w$, where $w\in\R^n$ denotes the feature weights. 
	We denote the error function for the iterate in the \texttt{SIM} Algorithm as $G(\cdot)$.
	%and the error at some time-step $k$ as $G_k(\cdot)$.
	
	Let the weight to which the regular Monte Carlo algorithm converges be called $\widetilde w$ and the best approximation to the value vector $V$ be $\Phi \widetilde w = \widetilde V$. Note that $\widetilde w = \min_{w} \sum_{i=1}^m d_i || \phi_i w - V_i||^2$. Similarly, let the weight vector to which we want \texttt{SIM} Algorithm to converge be $w^*$. Then $w^* = \min_{w} \sum_{i=1}^m d_i || \phi_i w - V_i||^2/||\phi_i||^2$. Let $V^* = \Phi w^*$ be our approximation of the value vector $V$.
	
	We denote the length of episode in Monte Carlo as $T$ with number of unique states seen as $\tau$
	
	%Let the weight to which 	value vector pair in Monte Carlo as $(w^{**},V^{**})$, in Scale Invariant Monte Carlo as $(w^*,V^M)$ and 	Scale Invariant TD(0) as $(w^N,V^N)$. 	In our algorithms, we denote the length of trajectory (length of episode in Monte Carlo) as $T$ with number of unique states seen as $\tau$
	
	\section{MAIN ALGORITHM AND ITS ANALYSIS}
	We outline our Total Projections (TP) method as a general method to find the scale invariant solution to an overdetermined system, through repeated projections. 
	Our main method is given in \hyperref[algo:TPAlgoforMonteCarloMain]{algorithm \ref{algo:TPAlgoforMonteCarloMain}}, where we run through a trajectory sampled from the stationary distribution. This method calls as a subroutine \hyperref[algo:TPAlgoforMonteCarloSubroutine]{algorithm \ref{algo:TPAlgoforMonteCarloSubroutine}}, for a one step stochastic weight update.
	%\footnote{$\widetilde{V}$ in the algorithm is a random sample who's conditional expectation, given the start state, is $V$}
	This method is inspired by Randomized Kaczmarz (our main modifications are highlighted in \hyperref[appendix:Differences to traditional Kaczmarz Algorithm]{appendix \ref{appendix:Differences to traditional Kaczmarz Algorithm}}).
	We speed up the algorithm through a novel step size method \hyperref[section:StepSizeSeq]{(section \ref{section:StepSizeSeq})} and momentum \hyperref[section: Choice of Momentum in Total Projections]{(section \ref{section: Choice of Momentum in Total Projections})}.
	
	\begin{figure}[!t]
		\begin{minipage}[t]{1\textwidth}
			\begin{minipage}[t]{0.48\textwidth}
				%\begin{algorithm}[t]
				\begin{algorithm}[H]
					\small
					\caption{\texttt{SIM} Algorithm for First-visit MC with curvature-step}
					\label{algo:TPAlgoforMonteCarloMain}	
					\begin{algorithmic}[1] % The number tells where the line numbering should start 
						\State {\bfseries Input:} $\Phi$, max Iterations
						\State {\bfseries Output:} $w^*$ - estimated ideal output weights
						\State Initialize weight vector $w_0$.
						\While{$||w_{k}-w_{k-1}|| > \varepsilon$} 
						\Statex \qquad $\triangleright$ where $\varepsilon$ is some small constant
						\State Value Function Estimate $\widetilde V_k = \Phi w_k$
						\State Let policy $\mu_k$ be $\epsilon$-greedy with respect to $\widetilde V_k$ \Statex \quad\quad $\triangleright$ here $\epsilon$ decays to 0.
						\State Get Trajectory as per policy $\mu_k$:
						\Statex \quad\quad $\triangleright$ Trajectory: $S_0,R_1,S_1,\dots,S_{T-1},R_T$
						\State $w_{k+1}\gets$  
						\Statex \qquad\texttt{TP}{($w_k,w_{k-1}$, Trajectory, $\Phi$)}
						\State $k\gets k+1$
						\EndWhile
						\vspace{0.05in}
						\State $w^* \gets w_k$, $V^*\gets \Phi w^*$
						\State $\mu \gets$ greedy policy with respect to $V^*$
						\vspace{0.05in}
						\State \textbf{return} $w^*,V^*,\mu^*$
						
					\end{algorithmic}
				\end{algorithm}
			\end{minipage}
			\hfill
			\begin{minipage}[t]{0.49\textwidth}
				\begin{algorithm}[H]
					\small
					\caption{\texttt{TP:} Total Projection Subroutine}
					\label{algo:TPAlgoforMonteCarloSubroutine}
					\begin{algorithmic}[1] 
						\State {\bfseries Input:} $w_k,w_{k-1}$, Trajectory, $\Phi$
						\State {\bfseries Output:} $w_{k+1}$
						\State Initialize $\widetilde{V}$, A, unique states counter $\tau$ to 0 
						\Statex  $\triangleright$ A indicates the discounted sum of rewards array
						\For {$t=$ last step of trajectory to first}
						\State $A(t) \gets R_{t} + \gamma \cdot A(t+1)$
						%\State $V(s_t) \gets A(t)$ where $s_t$ is the state at time $t$
						\EndFor
						\For {$t=$ first step of trajectory to last}
						\State \textbf{if} state $s_t$ seen for first time \textbf{then}:
						\State \quad $\widetilde{V}_\tau \gets A(t)$;\enspace $\phi_\tau \gets \Phi(s_t)$;\enspace$\T \gets \T +1$
						\EndFor
						
						\State $\eta\gets 1/k$ \qquad $\alpha \gets $ curvature step size
						\State $\beta \gets $ momentum multiplier \vspace{0.1cm}
						\State $\U_1 \gets \frac{1}{\T} \sum\limits_{i=0}^{\T-1} \dfrac{\phi_i^\tr w_k - \widetilde{V}_i}{||\phi_i||_2^2} \phi_i$ \vspace{0.1cm}
						\State $\U_2 \gets w_k - w_{k-1}$  \vspace{0.12cm}
						\State $w_{k+1} \gets w_k - \eta \alpha\U_1 + \beta \U_2$ 
						\State \textbf{return $w_{k+1}$}
						
					\end{algorithmic}
				\end{algorithm}

			\end{minipage}
		\end{minipage}
	\end{figure}
	
	The algorithm follows the same design of the regular Monte Carlo Algorithm for reinforcement learning in the outer loop \citep{suttonBarto}. This is indicated in Algorithm \ref{algo:TPAlgoforMonteCarloMain}. Here we run a trajectory as per an $\epsilon$-greedy policy with respect to the calculated weight vector $w_k$. We set the $\epsilon$ to be some sequence that decays to $0$. Asymptotically, this algorithm is greedy with respect to the approximated Value vectors $V_k = \Phi w_k$. In other words, at every state, it chooses the action that maximizes the one step reward plus the value at the next state.
	
	The above Algorithm runs the improved Algorithm \ref{algo:TPAlgoforMonteCarloSubroutine}, \texttt{TP} subroutine, which incorporate our major ideas. As noted in the discussion in Section \ref{Section: Our Contributions}, we use heavy-ball momentum and also use curvature-step with a decreasing multiplier $\eta = 1/k$.
	
	\begin{rem}
		Our main improvements are in the inner subroutine, Algorithm \ref{algo:TPAlgoforMonteCarloSubroutine}, of the $\SIM$. We envisage that this sub-routine can be utilized in other reinforcement learning algorithms under linear function approximation. The requirement is an ability to approximate the value function at each state, which in the case of Monte Carlo is the discounted sum of rewards from any state to the terminal state in the trajectory.
		
	\end{rem}
	
	\subsection{Analysis of Convergence}
	
	\begin{thm}
		\label{thm: monte carlo convergence}
		The stochastic approximation algorithm 
		\begin{equation}			
			w_{k+1} = w_k - \eta_k\dfrac{||TP_k(w_k)||}{||\Delta TP_k(w_k)||} TP_k(w_k) + \beta(w_{k}- w_{k-1})
		\end{equation}
		converges a.s. to 
		\begin{equation}
			w^*	:= \left[(\Phi^\tr N D N \Phi)^{-1}\Phi^\tr N D N \right] V
		\end{equation}
		where 
		%$TP_k(w_k) = \sum_{i=1}^\T \frac{(\phi_i^\tr w  - V_i)\phi_i}{||\phi_i||_2^2}$, 
		N is diagonal with $N_{(i,i)}$=$ \frac{1}{||\phi(i)||_2}$, $\beta \in (0,1)$, $\eta_k = 1/ k^p;\quad p\in(0.5,1]$
	\end{thm}
	
	To prove the above theorem, we first propose a simpler Theorem \ref{thm:convergence of monte carlo without momentum}, which does not involve the momentum term. We state and prove this below.
	
	%We will formally prove \hyperref[thm: monte carlo convergence]{this theorem} using stochastic approximation theory \cite{Borkar}
	
	%We first show the above theorem without momentum
	\begin{thm}
		\label{thm:convergence of monte carlo without momentum}
		$w_{k+1} = w_k - \alpha_k \cdot TP_k(w_k)$ 
		without momentum converges to $w^*$ (a.s)
	\end{thm}
	To prove convergence, we need to show that four conditions are satisfied. 
	
	The major claims that we use in this proof are the following:
	
	\textbf{Fact:} $V_i$'s are bounded. In other words, if $R_{max} = \max_{s,s'\in\S}[R_{ss'}]$, then $V_i \leq R_{max}/(1-\gamma)\enspace \forall i$ 
	
	\vspace{-0.06in}
	\textbf{Fact:} $\tau$ is bounded as it is the number of unique states

	Now let the filtration be $\F_k$=$\{w_0,\dots,w_k\}$.
	For the stochastic update equation in \hyperref[thm:convergence of monte carlo without momentum]{theorem \ref{thm:convergence of monte carlo without momentum}}, 
	let the expected update be $h_{k+1}(w_k) = \E  \left[TP_k(w_k) | \F_k\right]$. 
	Then, the update rule in standard form is $w_{k+1} = w_k -\alpha(h_{k+1}(w_k) +\M_{k+1})$
	
	\vspace{-0.05in}
	\begin{propn}
		\label{propn:MonteCarloA1}
		$h_{k+1}(w_k)$ is Lipschitz
	\end{propn}
	
	\vspace{-0.15in}
	\begin{proof}
		Proof in \hyperref[appendixSection:gradient of the error function is Lipschitz]{Appendix \ref{appendixSection:gradient of the error function is Lipschitz}} and \hyperref[appendixSection: h is Lipschitz]{appendix \ref{appendixSection: h is Lipschitz}}
	\end{proof}
	
	\begin{propn}
		\label{propn:MonteCarlo A2 step size sequence satisfies properties}
		The step size sequence $\{\alpha_i\}_{i=1}^\infty$ satisfy $\sltoinf \alpha_i = \infty$ and $\sltoinf \alpha_i^2 < \infty$
	\end{propn}

	\begin{proof}[Proof Sketch]
		This proceeds from our construction of the step size sequence in \hyperref[section:StepSizeSeq]{section \ref{section:StepSizeSeq}}. See \hyperref[appendixSection: alpha without momentum is square summable not summable]{appendix \ref{appendixSection: alpha without momentum is square summable not summable}} for full proof.
	\end{proof}

	\begin{propn}
		\label{propn:MonteCarloA3}
		$\{\M_{k}\}$ is a zero-mean martingale difference noise sequence
		
	\end{propn}
	
	\vspace{-0.1in}
	\begin{proof}
		We show this in \hyperref[appendixSection: Margingale difference sequence square integrable]{appendix \ref{appendixSection: Margingale difference sequence square integrable}}.
	\end{proof}
	
	\vspace{-0.1in}
	\begin{propn}
		\label{propn:MonteCarloA4}
		The iterates remain bounded almost surely. In other words, $\sup\limits_{k} w_k <\infty\enspace (a.s)$. 
	\end{propn}
	
	\begin{proof}
		First note that $V_i's$ are upper-bounded. Thus the estimates for the hyperplanes are upper-bounded. 
		Now, in a fully determined system, there is at least one, and at most $\binom{m}{n}$ 
		intersection points in $\R^n$ of the m hyperplanes. 
		Since each iteration brings us closer to at least one of these intersection points (by the Pythagoras theorem, as we are doing projections), and the intersection points are all bounded, the iterates are almost surely bounded
	\end{proof}\vspace{-0.05in}
	
	\begin{propn}
		\label{propn:monte carlo unique stable equilibrium point}
		Let $h(\cdot)$ be the function which our update equation tracks asymptotically, then the unique globally asymptotically stable equilibrium point for the limiting o.d.e given by $\dot{w}(t) = h(w(t))$ is given as $w^* = \left[(\Phi^\tr N D N\Phi)^{-1}\Phi^\tr N D N\right] V$
	\end{propn}
	\vspace{-0.1in}
	\begin{proof}
		We show this in \hyperref[appendixSection:Convergence of Scale Invariant monte carlo without momentum]{appendix \ref{appendixSection:Convergence of Scale Invariant monte carlo without momentum}}. 
	\end{proof}
	\vspace{-0.1in}
	\begin{proof}[Proof of Theorem \ref{thm:convergence of monte carlo without momentum}]
		From \hyperref[propn:MonteCarloA1]{propositions \ref{propn:MonteCarloA1}}, \hyperref[propn:MonteCarlo A2 step size sequence satisfies properties]{ \ref{propn:MonteCarlo A2 step size sequence satisfies properties}}, \hyperref[propn:MonteCarloA3]{ \ref{propn:MonteCarloA3}}, \hyperref[propn:MonteCarloA4]{ \ref{propn:MonteCarloA4}}, we satisfy the assumptions A1-A4 required to show convergence of a stochastic approximation equation \cite{Borkar}. Based on \hyperref[propn:monte carlo unique stable equilibrium point]{proposition \ref{propn:monte carlo unique stable equilibrium point}} we converge to the unique globally asymptotically stable equilibrium point given by $w^* =\left[(\Phi^\tr N D N\Phi)^{-1}\Phi^\tr N D N\right] V$\vspace{-0.07in}
	\end{proof}
	
	\vspace{-0.1in}
	\subsection{Convergence using Momentum}
	\label{section: Momentum in Scale Invariant Monte Carlo}
	
	\vspace{-0.1in}
	Momentum methods have been shown to converge by \cite{defossez2020convergence,reddiAdamConvergence}. Convergence under heavy-ball momentum has been shown by \citet{ghadimi2015global}. \citet{avrachenkov2020online} have used two-time scale methods to show convergence under momentum terms.
	We consider one such adaptation of these general techniques here.
	
	\begin{proof}[Proof of Theorem 1]
		\label{proof: Theorem 1}
		We cover the full proof in \hyperref[appendixsection: Showing convergence with momentum]{appendix \ref{appendixsection: Showing convergence with momentum}}. Here we provide two propositions (from \hyperref[appendixsection: Showing convergence with momentum]{appendix \ref{appendixsection: Showing convergence with momentum}}) that show that the final iterate is the same as the iterate without momentum, added with perturbation terms and a zero-mean martingale noise sequence. Given that the martingale noise and perturbations have zero expectation and are multiplied with a decaying scalar (that is square summable, but not summable), the convergence properties are the same as for the case without momentum.
	\end{proof}
	
	\begin{propn}
		The stochastic approximation equation with momentum can be rewritten as\\	
		$w_{k+1} - w_k = \alpha_k z_k$\\
		$z_i = z_{i-1} + \zeta_{(i,k)} \left[h(w_k) + \varepsilon_{(i,k)} + M_{(i,k)}\right] \quad \forall i \in [1,k]$\\
		$z_0 = \left[h(w_k) + M_{(0,k)}\right]$
		
		where $M_{(i,k)}$ are martingale difference noise, coefficients $\zeta_{(i,k)} =\beta^i \frac{\alpha_{k-i}}{\alpha_{k}}$ provide exponential decay, expected update $h(\cdot)$ converges to $w^*$ and $\varepsilon_{(i,k)}$ are perturbation terms.
	\end{propn}
	
	\begin{propn}
		The above set of equations collapse into the stochastic equation
		$w_{k+1} - w_k = \alpha_k[\h(w_k) + \vepsilondot_k + \Mdot_k]$
		where $\h(w_k)$ converges to $w^*$, $\{\vepsilondot_k\}$ are perturbation terms and $\{\Mdot_k\}$ are martingale difference noise terms.
	\end{propn}
	Note that the perturbation terms don't affect convergence and Martingales difference random variables have expectation $0$. Therefore convergence mainly depends on the first term. But the first term is the same as in \hyperref[thm:convergence of monte carlo without momentum]{Theorem \ref{thm:convergence of monte carlo without momentum}}. Therefore the iterates converge to the same point as in \hyperref[thm:convergence of monte carlo without momentum]{Theorem \ref{thm:convergence of monte carlo without momentum}}, even in the presence of momentum.

	\section{DISCUSSION ON STEP SIZE}
	\label{section:Explanation on Machinery}

	In this section we cover in detail our curvature step size and choice of momentum method.
	
	\vspace{-0.07in}
	\subsection{Adaptive step sizes for Total Projections Algorithm}
	\label{section:StepSizeSeq}
	Choice of step size is extremely important for ML practitioners. We propose a novel variation for a step size sequence.
	
	To achieve convergence for a stochastic approximation algorithm, we need the step size sequences $\{\alpha_k\}_{k=1}^\infty$ to be such that $\sum_{k=0}^\infty \alpha_k = \infty$ and $\sum_{k=0}^\infty \alpha_k^2 <\infty$ \citep{Borkar}. 
	To achieve this, our step size sequence takes the form $\eta_k \cdot \theta_k/||TP_k(\cdot)||$, where $\eta_k = 1/k^p;\enspace p\in(0.5,1]$. The second term $\theta_k$ is the term of interest currently, and the third term makes the existing update term $TP_k(\cdot)$ unit norm.
	
	\subsubsection{Idea for Curvature Step}
	
	\begin{figure}[hb]
		\vskip -0.1in
		\begin{center}
			% 			\centerline{\includegraphics[scale=.9]{images/CurvatureStepIllustration.png}}
			\centerline{\includegraphics[scale=.35]{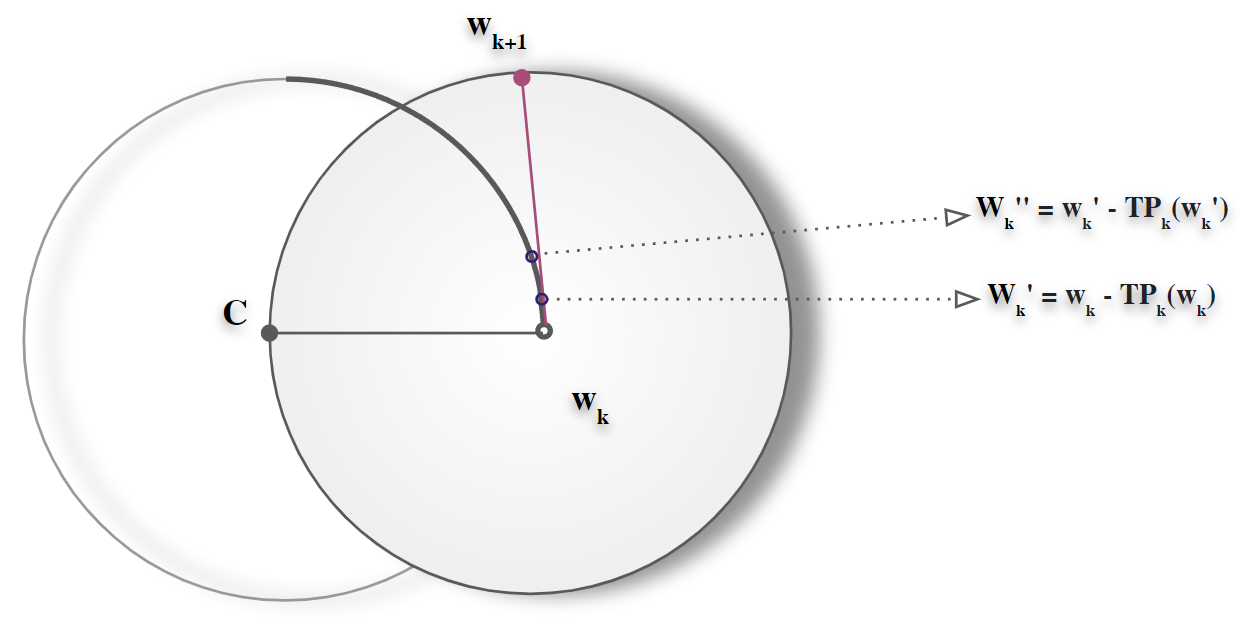}}
			\caption{Illustration for how the curvature-step works}
			\label{fig:Illustration for CurvatureStep}
		\end{center}
		\vskip -0.25in
	\end{figure}
	
	\hyperref[fig:Illustration for CurvatureStep]{Figure \ref{fig:Illustration for CurvatureStep}} illustrates the working of our curvature step on the step size based on the radius of the osculating circle. $w_k$ is the iterate, and C is the center of the circle formed by the osculating circle. We calculate the radius based on intermediate points $w_k'$ and $w_k''$, to finally get to point $w_{k+1}$
	
	%To motivate the discussion, we first observe that a stochastic gradient descent (with some subset of hyperplanes) induces a gradient field on the error we are descending. 
	%The local behavior of the field-curves can be studied by their Taylor expansion \cite{kuhnel2015differential}. 
	%The Taylor expansion up to 2D gives us a tangential and a normal direction. Such curves have a curvature at any given point, $\kappa$, which can be calculated from the Frenet Serret equations. When we assume the curvature is constant, we get an osculating circle. 
	%Our claim is that for our gradient curves, we can take in the direction of the tangent (our update), a step equal to the radius of the osculating circle.

	\vspace{-0.06in}
	\subsubsection{Estimating Radius of Osculating circle}
	\vspace{-0.06in}
	Let $w_k(t)$ be some stochastic gradient curve we are descending, with some subset of hyperplanes fixed. Then the curvature is given by $\kappa = ||w''(\cdot)||$, where w is parameterized to some unit vector in the space, and radius $R = 1/\kappa$. 
	
	\vspace{-0.03in}
	Note that that our updates, $TP_k(\cdot)$ are tangents to $w(t)$. Since our estimates are not unit parameterized, we need an appropriate change of scale (re-parametrization). In other words, we divide our estimate for tangent by $||TP_k(w_k)||$, to get the unit tangent. Similar re-scaling of our estimate for curvature yields $||TP_k(w_k)||^2$ in the denominator \cite{ChappersMathStackexchange}.

	Let $\Delta TP_k(w_k) = TP_k(w_k - TP_k(w_k)) - TP_k(w_k)$. Then, our guess for the second derivative is $||\Delta TP_k(w_k)||$, which after re-parametrization gives $||\Delta TP_k(w_k)||/||TP_k(w_k)||^2$. 
	Then we have $R = 1/\kappa = ||TP_k(w_k)||^2/||\Delta TP_k(w_k)||$. Thus:
	\begin{equation}
		\label{eqn:definition of curvature step}
		\begin{split}
			\theta_k &= \frac{||TP_k(w_k)||^2}{||\Delta TP_k(w_k)||}
		\end{split}
	\end{equation}

	Thus our update equation (without momentum) becomes: 
	
	\begin{equation}
		\begin{split}
			w_{k+1} &= w_k - \eta_k\theta_k \frac{TP_k(w_k)}{||TP_k(w_k)||} \\
			&= w_k - \eta_k\dfrac{||TP_k(w_k)||}{||\Delta TP_k(w_k)||} (TP_k(w_k))
		\end{split}
	\end{equation}

	We call the step size sequence $\alpha_k$ as \textit{curvature-step} sequence. We now provide a visual illustration and rationale for the curvature-step, for consideration alongside \hyperref[fig:Evidence for Curvature based Step size]{Figure \ref{fig:Evidence for Curvature based Step size}}.

	\comment{
		\vspace{-0.1in}
		\subsubsection{Rationale}
		\vspace{-0.05in}
		\begin{claim*}
			Minimizer $w^*$ of the error on $G(\cdot)$ is outside the $(n-1)$ sphere in $\R^n$ with center at the iterate $w_k$ and radius equal to radius of osculating circle
		\end{claim*}
		\vspace{-0.15in}
		\begin{proof}[Proof Sketch]
			We first note that the convergence point is an attractor node, and the convergence is Lyapunov stable, in expectation. The convergence rate here is slower along the lower eigenvalue-vectors of the system \cite{strogatz_NonLinear}. 
			Thus in \hyperref[fig:Illustration for CurvatureStep]{figure \ref{fig:Illustration for CurvatureStep}}, the iterate has more to move along other directions than along the first direction (i.e. along $TP(w_k)$). Thus the convergence point is outside the $n-1$ sphere centered at $w_k$ with radius R.
		\end{proof}
		\vspace{-0.15in}
	}

	\begin{figure}
		\centering
		\begin{subfigure}[t]{0.49\textwidth}
			%\centerline{\includegraphics[scale=0.64]{images/ComparisonofVariousMomentumTerms2.png}}
			\centerline{\includegraphics[scale=1]{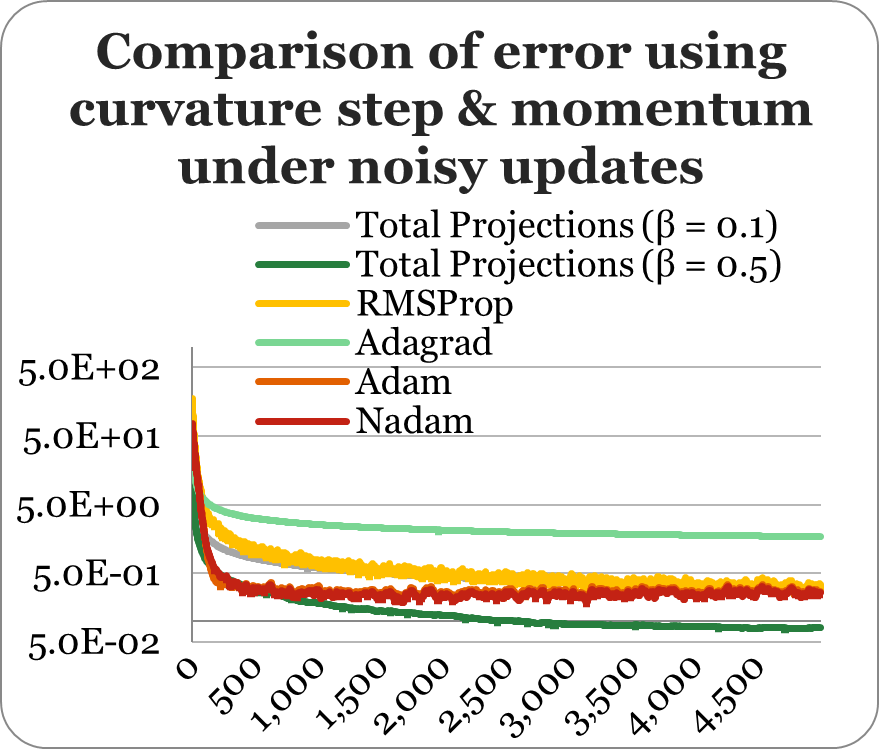}}
			\caption{Mean Error for 10 runs of Total Projections Algorithm with various momentum methods. $\Phi$ and $V$ are sampled uniformly from [-1,1]. (Remark: Adam and Nadam are nearly overlapping)}
			\label{fig:Comparison of Momentums}
		\end{subfigure}
		\hfill
		\begin{subfigure}[t]{0.49\textwidth}
			\includegraphics[scale=1]{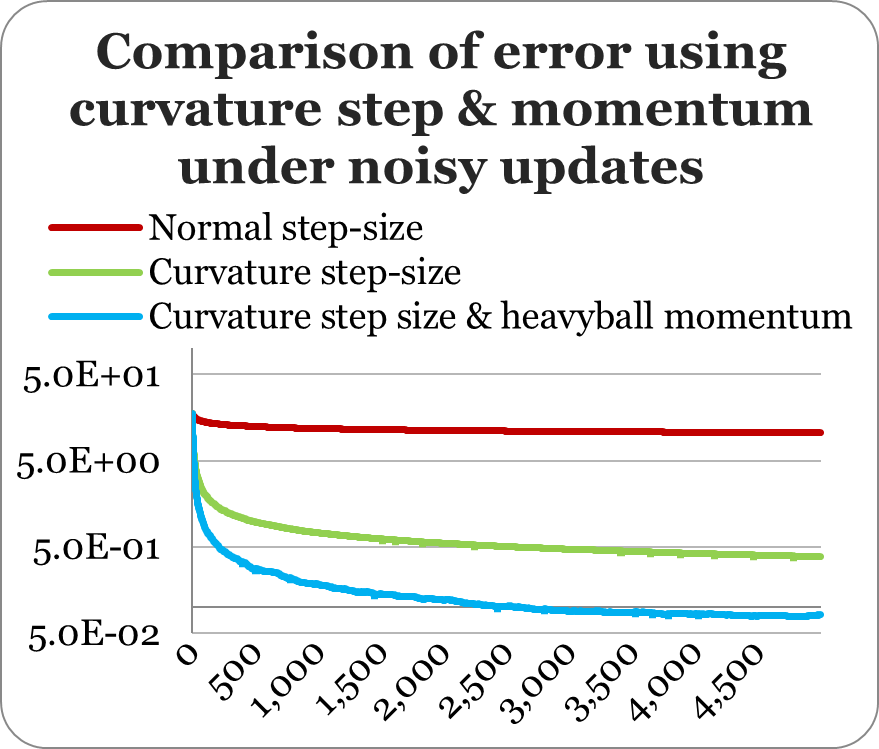}
			\caption{Comparison of error (distance from $w^*$) for plain Total Projections, with Curvature Step Size and Momentum. The figure illustrates the efficacy of our curvature-step size with momentum}
			\label{fig:Comparison Simulation}
		\end{subfigure}
		\vspace{-0.15in}
	\end{figure}

	\vspace{-0.in}
	\section{EXPERIMENTS}
	
	\vspace{-0.12in}
	We carried out simulations for systems with 25 states and 10 features (m=25,n=10) in the presence of noise to see efficacy of our proposed algorithm. We carry out two experiments. The first is to determine the momentum method to be used with our curvature step size method. The second experiment is to compare the efficacy of using a normal step size, using curvature step size with no momentum, and using curvature step size with (Polyak's) heavyball momentum.
	We outline these experiments below.
	
	% 	\subsection{Choice of Momentum}
	
	\label{section: Choice of Momentum in Total Projections}
	\vspace{-0.05in}
	\textbf{Momentum Method Used:} 
	Of the various momentum optimization methods used in gradient descent algorithms \cite{ruder2017overview}, our comparisons (\hyperref[fig:Comparison of Momentums]{figure \ref{fig:Comparison of Momentums}}) showed Heavy Ball momentum with $\beta = 0.5$ works best (reasons in \hyperref[appendixSection: Choice of constant momentum multiplier]{appendix \ref{appendixSection: Choice of constant momentum multiplier}}). We use this for our step size sequence. We notice that the decrease in error using some of the momentum methods is not monotonic, meaning that there could be bad updates that are amplified by the momentum method used. In this sense, the heavy-ball momentum is conservative, and ensures convergence so long as the original iterates converge, even in the presence of noise.
	We next look at whether using the heavy-ball momentum so chosen, we get better convergence rates than without using momentum, in the noisy setting.
	
	\vspace{-0.05in}
	\textbf{Advantage of using curvature step and momentum with noisy updates:} In \hyperref[fig:Comparison Simulation]{Figure \ref{fig:Comparison Simulation},} we compare convergence using (1) No Curvature Step size (2) Curvature Step and (3) Momentum. This is for the setting with $m=25,n=10$. We notice that the setting with curvature-step and heavy-ball momentum (with $\beta=0.5$) works best. 
	
	\vspace{-0.02in}
	\begin{rem}
		Our experiments show that the setting without curvature step has very poor convergence rates. This is in-line with the convergence rate of the Randomized Kaczmarz algorithm \citep{strohmerVershynin} which is inversely proportional to the square of the condition number of the linear system. 
		%Since the system is chosen randomly, we expect the condition number to go up with the size of the system. 
		Our experiments show that the curvature step reduces the dependence of convergence rate on the condition number (see figure \ref{fig:Evidence for Curvature based Step size}). We further note this reduced dependence continues in the noisy setting as well (see figure \ref{fig:Comparison Simulation}).
	\end{rem}
	
	\vspace{-0.15in}
	\section{CONCLUSIONS AND DISCUSSION}
	
	\vspace{-0.12in}
	In this work, we presented a scale-invariant version of the popular Monte Carlo algorithm for reinforcement learning. We gave a rationale for why Least Squares criterion fails in many instances, and the feature-scaled version should be used in the linear-function approximation setting. We then proposed a novel adaptive step size sequence based on the curvature of the path of convergence of the iterate $w_k$. We provided a convergence proof for this algorithm in the presence of momentum. Finally we experimentally validated our algorithm through simulations and showed that in the presence of noise we have significant speedups over the regular algorithm. Without noise, our algorithm in fact has exponentially faster convergence than the usual stochastic gradient update rule. 
	
	A possible extension of our work would be to use the proposed step size in the context of non-linear, non-convex settings. Further, we believe there is merit in applying our scale-invariant algorithm
	(rather than some least-squares variant) in various other linear settings where we wish to give equal importance to all data points -- irrespective of norm.
	
	%t iterative algorithm in other settings where all states are equally important .
	
	%We find there to be several scopes for extending our work. Immediately, one may like to test whether our proposed step size method works in case of non-linear, non-convex settings (say in deep learning contexts). Another possible extension is that our convergence analysis is asymptotic. One may want to look at finite sample analysis for the convergence as well. 
	
	%While our convergence rate should at least conform to the upper bounds suggested in the analysis of the Randomized Kaczmarz method of \citet{strohmerVershynin}, one may like to check how the adaptive step size speeds this up. 

	\clearpage

	\renewcommand\refname{REFERENCES}
	 
	\bibliography{References}
	\bibliographystyle{apalike}

	\clearpage

	\clearpage
	
	\clearpage
	%\appendix

	\appendix
	\appendixpage
	\addappheadtotoc

	% 	\section{Convergence Point of the Monte Carlo - Least Squares solution}
	\section{CONVERGENCE POINT OF THE MONTE CARLO - LEAST SQUARES SOLUTION}
	\label{appendixSection: Convergence Point of Monte Carlo}
	
	We now calculate the convergence point of the Monte Carlo algorithm. 
	The first visit Monte Carlo is an unbiased estimator for the value corresponding to states. Further, the updates under the Monte Carlo algorithm with linear function approximation correspond to a stochastic gradient descent on the least squares error function \cite{csaba,suttonBarto}.
	We will show here that the convergence point of the algorithm is given by $w^{**}=(\Phi^\tr \Phi)^{-1} (\Phi^\tr V)$
	\begin{propn}
		$w^{**}=(\Phi^\tr \Phi)^{-1} (\Phi^\tr V)$ and $V^{**} = \Phi (\Phi^\tr \Phi)^{-1} (\Phi^\tr V)$
	\end{propn}
	\begin{proof}
		
		\vspace{-0.15in}
		\begin{align}
			w^{**} = \argmin\limits_{w \in \R^n} \sum\limits_{s\in\S} (\phi_s^\tr w -V_s)^2
			\intertext{\vspace{-0.05in}Taking the derivative and setting it to 0 for the arg-min:}
			\frac{d}{dw} \left( \sum\limits_{s\in\S} (\phi_s^\tr w^{**} -V_s)^2 \right) &=0
			\intertext{\vspace{-0.05in}taking the derivative:}
			\sum\limits_{s\in\S} 2\phi_s (\phi_s^\tr w^{**} - V_s) &= 0\nonumber\\
			\left( \sum\limits_{s\in\S} \phi_s \phi_s^\tr\right) w^{**} &= \sum\limits_{s\in\S} V_s\phi_s 
			\intertext{$\left( \sum\limits_{s\in\S} \phi_s \phi_s^\tr\right) = \Phi\cdot \Phi^\tr$ and $\left( \sum\limits_{s\in\S} \phi_s V_s \right) = \Phi^\tr\cdot V$. Thus:}
			\Phi^\tr \Phi w^{**} &= \Phi^\tr V\nonumber\\
			\intertext{Then we have:}
			w^{**}&=(\Phi^\tr \Phi)^{-1} (\Phi^\tr V)\nonumber\\
			V^{**} = \Phi w^{**} &= \Phi(\Phi^\tr \Phi)^{-1} (\Phi^\tr V)
		\end{align}
	\end{proof}
	Thus in the case of Least Squares we have the solution given by $V^{**}=\Phi w^{**}=\Phi (\Phi^\tr \Phi)^{-1} (\Phi^\tr V)$
	
	In \hyperref[fig:Least Squares]{figure \ref{fig:Least Squares}}, we illustrate the $\R^m$ perspective of the least squares solution. The least squares solution is a projection onto the column space of $\Phi$. In other words, the solution is the point on the column space of $\Phi$, which is at least distance from V. Our claim is that such a solution may be unduly affected by rows which have large feature-norm. 
	
	For comparison, this solution can be compared with \hyperref[fig:TotalProj]{figure \ref{fig:TotalProj}}, where we illustrate in $\R^n$ why distances to hyperplanes might be a scale invariant solution, which is unaffected by the feature norms.
	
	\begin{figure}[ht]
		\vskip 0.2in
		\begin{center}
			\includegraphics[scale=0.52]{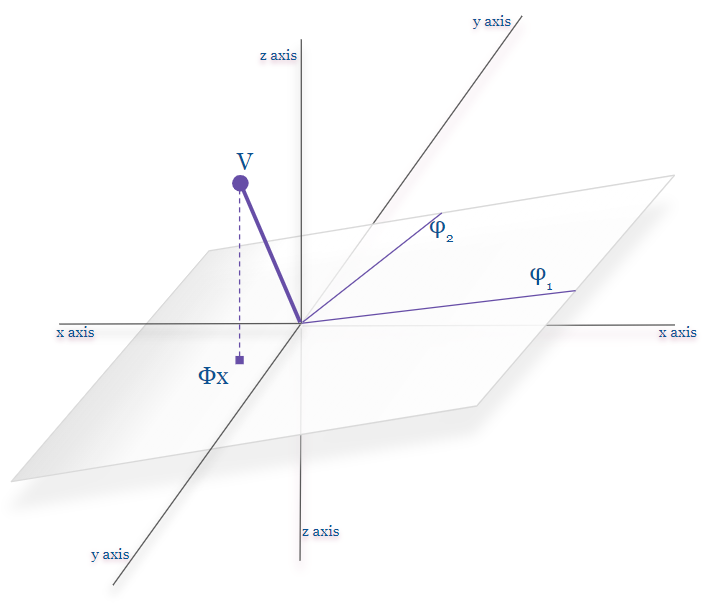}
			\caption{Illustration of the Least Squares solution as a projection}
			\label{fig:Least Squares}
		\end{center}
		\vskip -0.2in
	\end{figure}
	
	\newpage
	%\section{Differences in Total Projections from the traditional Kaczmarz Algorithm}
	\section{DIFFERENCES BETWEEN OUR ALGORITHM AND THE KACZMARZ ALGORITHM}
	
	\label{appendix:Differences to traditional Kaczmarz Algorithm}
	Our algorithm is a variation on the Randomized Kaczmarz algorithm described in \cite{strohmerVershynin}. We note the major differences below
	\begin{enumerate}
		\item The Randomized Kaczmarz algorithm samples the hyperplanes with a probability proportional to the square of the feature-norm, viz $||\phi_s||_2^2$ \cite{strohmerVershynin}. This approach has been criticized in literature \cite{strohmer_comments_2009}. (In our own simulations, this sampling did not provide any benefits). To sample proportional to the feature-norm square of the states, one needs to know the features-norms of all states, which may not be possible
		
		\item In the RL context, obtaining all possible features ab-initio is difficult, and so is sampling as per feature-norm square.
		Natural sampling would be as per the stationary distribution of the ergodic Markov Chain and we allow for this.
		
		\item The original Randomized Kaczmarz method was meant for a fully determined $Ax=b$ system. Therefore, in the original setup, the iterates lie on hyperplanes onto which one projects.
		On the other hand, our iterates don't lie on any hyperplane. This makes it easier to identify the sequence of iterates with a gradient field (of our error function). 
		
		\item We obtain major speedups (up to a few orders of magnitude) over the regular Kaczmarz method due to our usage of \textit{momentum} and step size based on \textit{radius of osculatory-circle}.
		
		\item Our formulation makes the algorithm directly a gradient descent on the error function $\sum\limits_{s\in\S} d_s \left[\dfrac{\phi_s^\tr w - V_s}{||\phi_s||}\right]^2$ where $d_s$ are some positive weights corresponding the hyperplanes $\H_s \equiv \phi_s^\tr w - V_s$. For example, $d_s =\frac{1}{|\S|}\enspace\forall s\in\S$ may correspond to a uniform sampling. Another example is where $d_s = \pi_s$ where $\pi$ is the stationary distribution corresponding to the Transition Matrix of a Markov Chain. 
	\end{enumerate}

	\newpage
	%\section{Properties of the Total Projections Operation}
	\section{PROPERTIES OF THE TOTAL PROJECTION (TP) OPERATOR}
	\label{appendixSection:TP Algorithm non-stochastic properties}
	In this section, we will consider properties of the Total Projections operation $TP(w) = \dfrac{1}{|\S|} \slsins d_s\left[ \dfrac{\phi_s^\tr w - V_s}{2||\phi_s||_2^2}\right]\phi_s$ and the error function $G(w) = \dfrac{1}{|\S|} \slsins d_s\left[ \dfrac{\phi_s^\tr w - V_s}{2||\phi_s||_2}\right]^2$ such that $\slsins d_s \leq |\S|$ where $d_s$ are some positive weights attached to hyperplanes $\H_s\equiv \phi_s^\tr w - V_s$. 
	
	The properties shown below hold in general for positive weights $\{d_s\}_{s\in\S}$ as long as $\sum_{s\in\S}d_s \leq |\S|$. But it may be worthwhile to consider what these positive weights may be. One example set of weight is $d_s = 1\enspace \forall s\in\S$, which may be considered as uniform weights. Another weight set is $d_s = \pi_s |\S|$ where $\pi_s$ is the probability of occurrence of state $s$ in the stationary distribution, which will be of interest to us in our algorithms.

	We will now show the following properties in the section numbers given:
	
	\begin{enumerate}[itemsep=-0.05in,leftmargin=1in,label={C.\arabic*. }]
		\item $TP(\cdot) = \nabla_w G(\cdot)$
		\item $G(\cdot)$ is convex
		\item $G(\cdot)$ is strongly convex
		\item $\nabla G(\cdot)$ is a Lipschitz function
		\item $Hess\left( G(\cdot)\right)$ is bounded above
		\item The batch version of the Total Projections algorithm converges
		\item Conditions on the step size of the total projection algorithm
		\item Convergence Rate of the Total Projections Algorithm
	\end{enumerate}

	\subsection{Total Projection is a gradient descent on the error function}
	
	\begin{appendixpropn}
		\label{appendixpropn:total projections is a gradient descent on the error function}
		Let $$TP(w) =  \dfrac{1}{|\S|}\slsins d_s\left[ \dfrac{\phi_s^\tr w - V_s}{2||\phi_s||_2^2}\right]\phi_s$$
		
		Then $TP(w) = \nabla_{w} G(w)$
	\end{appendixpropn}
	\begin{proof}
		We obtain this by just differentiating $G(\cdot)$ with respect to w
	\end{proof}
	
	\hyperref[fig:TotalProj]{Figure \ref{fig:TotalProj}} is an illustration of the convergence point of the Total Projections Algorithm. We have three hyperplanes in $\R^2$ and we attempting to find a $w$ such that $w$ is the point that minimizes the total sum of squares of distances to these hyperplanes. Note that hyperplanes are scale invariant in the sense, $\phi_s ^\tr w = V_s$ is the same hyperplane as $c \cdot\phi_s ^\tr w = c\cdot V_s$ for any arbitrary $c\in\R$. Thus our solution remains invariant under a multiplication of any row by a constant $c\in\R$
	
	\begin{figure}[ht]
		\vskip 0.2in
		\begin{center}
			\includegraphics[scale=0.45]{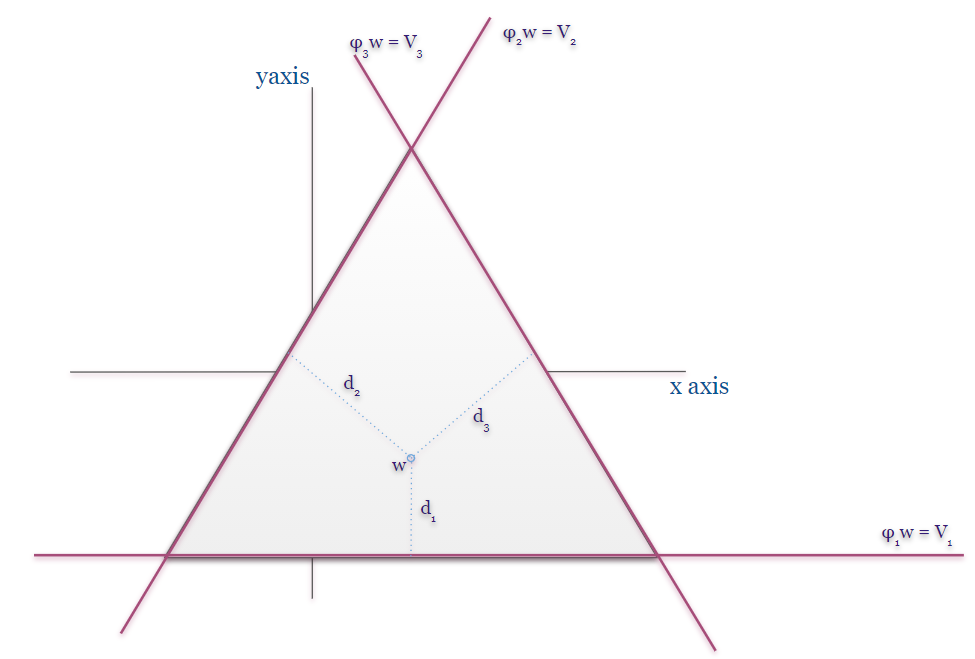}
			\caption{Illustration of the Point that minimizes the sum of squares of distances to hyperplanes}
			\label{fig:TotalProj}
		\end{center}
		\vskip -0.2in
	\end{figure}

	\subsection[The error function for Monte Carlo is convex]{$G(\cdot)$ is convex}
	
	In this subsection, we will show:
	\begin{enumerate}[itemsep=-0.05in,leftmargin=1in,label=(\alph*)]
		\item $G(w) =\dfrac{1}{|\S|} \slsins d_s\left[ \dfrac{\phi_s^\tr w - V_s}{2||\phi_s||_2}\right]^2$ is convex
		
		using:
		\item $\phi \phi^\tr$ is a positive semi definite matrix for all $\phi \in \R^n$
	\end{enumerate}
	
	\begin{appendixpropn}
		\label{appendixpropn:2norm error metric is convex}
		$G(\cdot)$ is convex in $\R^n$
	\end{appendixpropn}
	
	\begin{proof}
		We have already seen in \hyperref[appendixpropn:total projections is a gradient descent on the error function]{Section \ref{appendixpropn:total projections is a gradient descent on the error function}} that $\nabla_{w} G(w) = \dfrac{1}{|\S|}\slsins d_s\left[ \dfrac{\phi_s^\tr w - V_s}{2||\phi_s||_2^2}\right]\phi_s$. Now we have 
		\begin{enumerate}
			\item $\R^n$ is a convex set
			\item $G(\cdot)$ is twice differentiable
		\end{enumerate}
		
		Thus it is sufficient to show that $Hess(G(\cdot))$ is positive semi-definite. 
		
		\begin{align*}
			\nabla_w \enspace G(w) & = \dfrac{1}{|\S|}\slsins d_s\left[ \dfrac{\phi_s^\tr w - V_s}{2||\phi_s||_2^2}\right]\phi_s
		\end{align*}
		
		Then,
		\begin{align*} 
			Hess(G(\cdot)) &= \nabla_w (\nabla_w \enspace G(w))\\
			&=\nabla_w\left(\dfrac{1}{|\S|} \slsins d_s\left[ \dfrac{\phi_s^\tr w - V_s}{2||\phi_s||_2^2}\right]\phi_s\right)\\
			&=\dfrac{1}{|\S|}\slsins d_s \nabla_w\left(  \dfrac{(\phi_s^\tr w  - V_s)\phi_s}{||\phi_s||_2^2} \right)\\
			&=\dfrac{1}{|\S|}\slsins d_s \left(  \dfrac{\nabla_w (\phi_s^\tr w  - V_s)\phi_s}{||\phi_s||_2^2} \right)\\
			&=\dfrac{1}{|\S|}\slsins d_s \left(  \dfrac{ \phi_s \nabla_w (\phi_s^\tr w)}{||\phi_s||_2^2} \right)\\
			&=\dfrac{1}{|\S|}\slsins d_s \left(  \dfrac{\phi_s \phi_s^\tr}{||\phi_s||_2^2} \right)
		\end{align*}
		By \hyperref[appendixpropn:phiphi.t is psd]{proposition \ref{appendixpropn:phiphi.t is psd}}, $Hess(G(\cdot)) $ is the sum of $|S|$ positive definite matrices, weighted by some positive coefficients $\dfrac{\pi_s}{|\S|\cdot||\phi_s||_2^2}$. Thus $Hess(G(\cdot))$ is positive semi definite. Thus $G(\cdot)$ is a convex function

		An alternate method to show $G(\cdot)$ is convex, would be to show that $w^\tr [Hess(G(w))] w\geq 0$. 	
		We showed earlier that $Hess(G(\cdot)) = \dfrac{1}{|\S|}\slsins d_s \left(  \dfrac{\phi_s \phi_s^\tr}{||\phi_s||_2^2} \right)$
		Let $w\in \R^n$. Then we have 
		\begin{align*} 
			w^\tr [Hess(G(w))] w  & = w^\tr \dfrac{1}{|\S|}\slsins d_s \left(  \dfrac{\phi_s \phi_s^\tr}{||\phi_s||_2^2} \right) w \tag*{where $\phi_s \neq \0 \quad \forall i$}\\
			& = \dfrac{1}{|\S|}\slsins d_s \left(  \dfrac{w^\tr \phi_s \phi_s^\tr w}{||\phi_s||_2^2} \right)\\
			& = \dfrac{1}{|\S|}\slsins d_s \left(  \dfrac{(\phi_s^\tr w)^\tr (\phi_s^\tr w)}{||\phi_s||_2^2} \right)\\
			& = \dfrac{1}{|\S|}\slsins d_s \left(  \dfrac{(y_i)^\tr (y_i)}{||\phi_s||_2^2} \right) \tag*{where $y_i = \phi_s^\tr w$}\\
			&= \dfrac{1}{|\S|}\slsins d_s \left(  \dfrac{||y_i||_2^2}{||\phi_s||_2^2} \right)\\
			&\geq 0 \tag*{$\forall w \in \R^n$}
		\end{align*}
		This shows that $G(\cdot)$ is convex.

	\end{proof}
	\begin{appendixpropn}
		\label{appendixpropn:phiphi.t is psd}
		$\phi \phi^\tr$ is a positive semi definite matrix for all $\phi \in \R^n$
	\end{appendixpropn}
	\begin{proof}
		Let $\Phi = \phi \phi^\tr$ (for this proposition). Then to prove $\Phi$ is psd, it is sufficient to show $w^\tr \Phi w \geq 0 \quad \forall w \in \R^n$
		
		\begin{align*} 
			w^\tr \Phi w &= w^\tr \phi \phi^\tr w\\
			&=(\phi^\tr w)^\tr (\phi^\tr w)\\
			&= y^\tr y \text{\qquad where $y = \phi^\tr w$}\\
			&= ||y||_2^2 \text{\qquad 2-norm squared of y}\\
			&\geq 0
		\end{align*}
		Thus $\Phi = \phi \phi^\tr$ is positive semi-definite for all $a \in \R^n$
	\end{proof}

	\subsection[The error function for Monte Carlo is strongly convex, and thereby strictly convex]{$G(\cdot)$ is strongly convex and thereby strictly convex}
	
	In this subsection, we will show:
	\begin{enumerate}[itemsep=-0.05in,leftmargin=1in,label=(\alph*)]
		\item $G(\cdot)$ is strongly convex when rank($\Phi$) = n
		\item $G(\cdot)$ is strictly convex when rank($\Phi$) = n
	\end{enumerate}

	\begin{appendixpropn}
		\label{appendixpropn:G is strongly convex}
		$G(\cdot)$ is strongly convex if $\Phi \in R^{m\times n}$ has rank n
	\end{appendixpropn}
	\begin{proof}
		If $\Phi$ has rank n, then the vectors $\{\phi_s \}_{s\in\S}$ span $\R^n$. Then we have to show that if $\lambda_{min}$ is the least eigenvalue of $Hess(G(\cdot)) = \dfrac{1}{|\S|}\slsins d_s \left(  \dfrac{\phi_s \phi_s^\tr}{||\phi_s||_2^2} \right)$, then $\lambda_{min} >0$. We show this as follows:

		$\phi_s\phi_s^\tr$ is a rank 1 symmetric matrix. Symmetric matrices have real eigen values. Further, 
		\begin{align*}
			(\phi_s^\tr w)^2 &\geq 0\tag*{$\forall \phi_s,w\in \R^n$}\\
			(\phi_s^\tr w)(\phi_s^\tr w) &\geq 0\\
			(w^\tr \phi_s)(\phi_s^\tr w) &\geq 0\\
			w^\tr (\phi_s \phi_s^\tr) w &\geq 0\\
			w^\tr \left(\dfrac{\phi_s \phi_s^\tr}{||\phi_s||_2^2}\right) w &\geq 0\\
			\dfrac{1}{|\S|}\slsins d_s w^\tr \left(\dfrac{\phi_s \phi_s^\tr}{||\phi_s||_2^2}\right) w &\geq 0\\
			w^\tr \left(\dfrac{1}{|\S|}\slsins d_s \left(\dfrac{\phi_s \phi_s^\tr}{||\phi_s||_2^2}\right)\right) w &\geq 0\\
		\end{align*}
		
		This means the eigenvalues of $\dfrac{1}{|\S|}\slsins d_s \left(\dfrac{\phi_s \phi_s^\tr}{||\phi_s||_2^2}\right)$ are non-negative. It remains to be shown that the no eigenvalue is equal to 0. This is true as if some eigenvalue is equal to 0, then for the corresponding eigenvector, say $w$,
		\begin{align*}
			w^\tr \left(\dfrac{1}{|\S|}\slsins d_s \left(\dfrac{\phi_s \phi_s^\tr}{||\phi_s||_2^2}\right)\right) w &= 0\\
			\dfrac{1}{|\S|}\slsins d_s w^\tr \left(\dfrac{\phi_s \phi_s^\tr}{||\phi_s||_2^2}\right) w &= 0\\
			w^\tr \left(\dfrac{\phi_s \phi_s^\tr}{||\phi_s||_2^2}\right) w &= 0 \quad \forall \phi_s \in A\\
			w^\tr \phi_s \phi_s^\tr w &= 0 \quad \forall \phi_s \in A\\
			(\phi_s^\tr w)^2 &= 0 \quad \forall \phi_s \in A\\
			\phi_s^\tr w &= 0 \quad \forall \phi_s \in A\\
		\end{align*}
		But this is a contradiction as $\{\phi_s \}_{i=1}^n$ is a spanning set for $\R^n$. Thus minimum eigenvalue of $Hess(G(\cdot)) = \dfrac{1}{|\S|}\slsins d_s \left(  \dfrac{\phi_s \phi_s^\tr}{||\phi_s||_2^2} \right)$ is greater than 0.
		Thus by definition of strong convexity we have that $G(\cdot)$ is strongly convex when $\Phi \in \R^{m\times n}$ has rank n.

	\end{proof}
	\begin{appendixpropn}
		\label{appendixpropn: G is mu strongly convex}
		Let $\lambda_{min}$ be the least eigen value of $\dfrac{1}{|\S|}\slsins d_s \left(  \dfrac{\phi_s \phi_s^\tr}{||\phi_s||_2^2} \right)$. Then, 
		$G(\cdot)$ is $\mu-$strongly convex where $\mu = \lambda_{min}$ 
	\end{appendixpropn}
	\begin{proof}
		Note that we can show $\mu$-strongly convex when we show the following.
		Consider
		\begin{align*}
			w^\tr (Hess(G(\cdot))-\mu I) w  & = w^\tr \left(\dfrac{1}{|\S|}\slsins d_s \left(  \dfrac{\phi_s \phi_s^\tr}{||\phi_s||_2^2} \right) -\mu I \right)w \tag*{where $\phi_s \neq \0 \quad \forall i$}\\
			& = w^\tr \left(\dfrac{1}{|\S|}\slsins d_s  \dfrac{\phi_s \phi_s^\tr}{||\phi_s||_2^2} \right)w -\mu w^\tr w \\
			\intertext{If $\lambda_{min}$ is the least eigenvalue of $\left(\dfrac{1}{|\S|}\slsins d_s  \dfrac{\phi_s \phi_s^\tr}{||\phi_s||_2^2} \right)$,  $\lambda_{min} > 0$ as we have already shown}
			w^\tr (Hess(G(\cdot))-\mu I) w & \geq \lambda_{min}w^\tr w - \mu w^\tr w\\
			&\geq (\lambda_{min} - \mu)||w|_2^2
			\intertext{Thus $\exists\ \mu \in [0,\lambda_{min}) $ such that}
			w^\tr (Hess(G(\cdot))-\mu I) w  &> 0
		\end{align*}
		Thus we see that $G(\cdot)$ is $\mu$ strongly convex where $\mu = \lambda_{min}$ 
	\end{proof}

	\begin{appendixpropn}
		\label{appendixpropn:G is strictly convex}
		$G(\cdot)$ is strictly convex
	\end{appendixpropn}
	\begin{proof}
		Strict convexity is a subset of strong convexity. Thus G is strictly convex.
	\end{proof}

	\subsection[The gradient of the error function is Lipschitz]{$\nabla G(\cdot)$ is a Lipschitz function}
	\label{appendixSection:gradient of the error function is Lipschitz}
	
	In this subsection, we will show:
	\begin{enumerate}[itemsep=-0.05in,leftmargin=1in,label=(\alph*)]
		\item $\nabla_w G(\cdot)$ is Lipschitz continuous with Lipschitz constant equal to $1$
		
		using
		\item $\dfrac{\phi_s\phi_s^\tr}{||\phi_s||_2^2}$ has one eigenvalue 1 and rest $n-1$ eigenvalues 0.
	\end{enumerate}

	\begin{appendixpropn}
		\label{appendixpropn: max eigenval of phiiphii.tr by norm phi sq is 1}
		Let $\phi_s\neq \0$. Then, $\mathcal{V}_s = \dfrac{\phi_s\phi_s^\tr}{||\phi_s||_2^2}$ has one eigenvalue 1 and rest $n-1$ eigenvalues 0.
	\end{appendixpropn}
	\begin{proof}
		\textbf{Claim 1: eigenvalues are real and $\mathcal{V}_s$ is p.s.d:} Let $\mathcal{V}_s = \dfrac{\phi_s\phi_s^\tr}{||\phi_s||_2^2}$. Then, $\mathcal{V}_s$ is symmetric thus has real eigen values (\cite{axler}). The second part follows from \hyperref[appendixpropn:phiphi.t is psd]{proposition \ref{appendixpropn:phiphi.t is psd}}
		
		\textbf{Claim 2: $\mathcal{V}_s$ is a rank 1 matrix:} We note that the rank of $v v^\tr$ = 1 for any $v\neq \0, v\in \R^n$. This is because the rank is the dimension of the column space of the matrix. Since the columns of $v v^\tr$ are all scalar multiples of $v$, rank is 1
		
		\textbf{Claim 3: $\dfrac{\phi_s}{||\phi_s||_2}$ is an eigenvector of $\dfrac{\phi_s\phi_s^\tr}{||\phi_s||_2^2}$ with eigenvalue 1:}
		Let $\nu_s = \dfrac{\phi_s}{||\phi_s||_2}$. Then let $\nu_s\nu_s^\tr = \mathcal{V}_s$.
		Then we have to show $\nu_s$ is an eigen vector of $\mathcal{V}_s = \nu_s\nu_s^\tr$. But this is easy to see. $\mathcal{V}_s \nu_s = \nu_s \nu_s^\tr\nu_s = \nu_s \frac{\phi_s^\tr\phi_s}{||\phi_s||_2^2} = \nu_s$. Thus $\nu_s$ is an eigen vector of $\mathcal{V}_s$ with eigen value 1
		
		\textbf{Claim 4: The other eigenvectors are orthogonal to eigenvector with eigenvalue 1}
		First we note that $\nu_s$ is the only eigenvector of $\mathcal{V}_s$ with eigen value 1. Then we show in general that in a real symmetric matrix, eigenvectors with distinct eigenvalues are orthogonal.
		
		Let $\nu_1$ and $\nu_2$ be two eigenvectors of $\mathcal{V}_s$ with \textit{distinct} eigenvalues $\mu_1$ and $\mu_2$. Then $\mathcal{V}_s\nu_1 = \mu_1 \nu_1$ and $\mathcal{V}_s\nu_2 = \mu_2 \nu_2$. Consider $\mu_1 \nu_2^\tr \nu_1$. This is equal to $ \nu_2^\tr (\mu_1\nu_1) = \nu_2^\tr \mathcal{V}_s \nu_1 = \nu_2^\tr \mathcal{V}_s^\tr \nu_1 = (\mathcal{V}_s\nu_2)^\tr \nu_1 = (\mu_2\nu_2)^\tr \nu_1 = \mu_2 \nu_2^\tr\nu_1$. Thus $\mu_1 \nu_2^\tr\nu_1 = \mu_2 \nu_2^\tr\nu_1$ for distinct $\mu_1,\mu_2$, implying that $\nu_2^\tr\nu_1=0$, or in other words $\nu_1$ and $\nu_2$ are orthogonal

		\textbf{Claim 5: eigenvalue 0 has a multiplicity of $n-1$:}
		It can be shown \cite{axler} that a rank 1 matrix has at most 1 non-zero eigenvalue and eigenvalue 0 with multiplicity $n-1$ as follows.
		
		First we note that there are n eigenvectors for $\mathcal{V}_s = \nu_s\nu_s^\tr$ in $\R^n$. We have found one eigenvector $\nu_s$ with eigenvalue 1. We have also shown that all other eigenvectors are orthogonal to $\nu_s$. Consider any eigenvector $\nu$ orthogonal to $\nu_s$. Then $\nu_s^\tr\nu = 0$. Now consider $\mathcal{V}_s \nu = (\nu_s \nu_s^\tr) \nu = \nu_s (\nu_s^\tr\nu) = \nu_s \cdot 0 = 0$. Thus for all $n-1$ eigenvectors orthogonal to $\nu_s$, eigenvalue is 0

		Thus $\dfrac{\phi_s\phi_s^\tr}{||\phi_s||_2^2}$ has eigenvalue 1 with multiplicity 1, and eigenvalue 0 with multiplicity $n-1$
		
	\end{proof}
	
	Now we are ready to show the Lipschitz property of $\nabla G(\cdot)$
	
	\begin{appendixpropn}
		\label{appendixpropn:grad G is Lipschitz continuous}
		$\nabla G(\cdot)$ is Lipschitz continuous
	\end{appendixpropn}
	
	\begin{proof}
		We already showed that:
		\begin{align*} 
			\nabla G(w)  &= \dfrac{1}{|\S|}\slsins d_s \dfrac{(\phi_s^\tr w  - V_s)\phi_s}{||\phi_s||_2^2}\\
			\intertext{Then we have}
			\nabla G(w)-\nabla G(y)  &= \left( \dfrac{1}{|\S|}\slsins d_s \dfrac{(\phi_s^\tr w  - V_s)\phi_s}{||\phi_s||_2^2} \right) - \left(\dfrac{1}{|\S|}\slsins d_s \dfrac{(\phi_s^\tr y  - V_s)\phi_s}{||\phi_s||_2^2} \right)\\
			&= \left( \dfrac{1}{|\S|}\slsins d_s \dfrac{(\phi_s^\tr (w-y))\phi_s}{||\phi_s||_2^2} \right)\\
			&= \left( \dfrac{1}{|\S|}\slsins d_s \dfrac{\phi_s(\phi_s^\tr (w-y))}{||\phi_s||_2^2} \right)\\
			&= \left( \dfrac{1}{|\S|}\slsins d_s \dfrac{\phi_s \phi_s^\tr }{||\phi_s||_2^2} \right)(w-y)\\
			\intertext{Since max eigen value of $\dfrac{\phi_s \phi_s^\tr }{||\phi_s||_2^2}$ is 1}
			||\nabla G(w)-\nabla G(y)|| &\leq (\dfrac{1}{|\S|}\slsins d_s )||w-y||\\
			||\nabla G(w)-\nabla G(y)|| &\leq \dfrac{|\S|}{|\S|} ||w-y||=||w-y||
		\end{align*}
		Thus the function $G(\cdot)$ is Lipschitz where the Lipschitz constant, $L\leq1$

	\end{proof}
	
	\subsection[The Hessian of the error function is bounded above]{The Hessian of $G(\cdot)$ is bounded above}
	\label{appendixsection:Hessian of G is bounded above}
	In this subsection, we will show:
	\begin{enumerate}[itemsep=-0.05in,leftmargin=1in,label=(\alph*)]
		\item The Hessian of $G(\cdot)$ is bounded above, or 
		
		$Hess G(\cdot)\preceq I$\quad where $I$ is the identity matrix
	\end{enumerate}
	
	\begin{appendixpropn}
		\label{appendixpropn: Hessian of G is bounded above}
		$Hess G(\cdot)\preceq I$\quad where $I$ is the identity matrix
	\end{appendixpropn}
	\begin{proof}
		The proposition is equivalent to showing $w^\tr Hess(G(w)) w \leq ||w||_2^2$
		
		Using $Hess(G(\cdot)) = \dfrac{1}{|\S|}\slsins d_s \left[ \dfrac{\phi_s \phi_s^\tr}{||\phi_s||_2^2} \right]$, for some $w\in \R^n$, we have 
		\begin{align*} 
			w^\tr Hess(G(w)) w  & = w^\tr \dfrac{1}{|\S|}\slsins d_s \left[  \dfrac{\phi_s \phi_s^\tr}{||\phi_s||_2^2} \right] w \tag*{where $\phi_s \neq \0 \quad \forall i$}\\
			& = \dfrac{1}{|\S|}\slsins d_s w^\tr \left[  \dfrac{ \phi_s \phi_s^\tr}{||\phi_s||_2^2} \right] w\\
			\intertext{Since the maximum eigenvalue of $\dfrac{ \phi_s \phi_s^\tr}{||\phi_s||_2^2} = 1$ by \hyperref[appendixpropn: max eigenval of phiiphii.tr by norm phi sq is 1]{proposition \ref{appendixpropn: max eigenval of phiiphii.tr by norm phi sq is 1}},}
			w^\tr Hess(G(w)) w &= \dfrac{1}{|\S|}\slsins d_s w^\tr \left(  \dfrac{ \phi_s \phi_s^\tr}{||\phi_s||_2^2} w \right)\\
			& \leq \dfrac{1}{|\S|}\slsins d_s w^\tr w \\
			&= \dfrac{1}{|\S|}\slsins d_s ||w||_2^2 \\
			&\leq ||w||_2^2 \\
		\end{align*}
	\end{proof}
	
	\subsection{The batch version of the Total Projections algorithm converges}
	Now we proceed to prove convergence of the batch version (non stochastic version) of the TP algorithm. We have already shown $G(\cdot)$ is convex. Thus, we know that it has a unique optimum point. Thus if our algorithm converges to $\textit{some}$ optimum, it is guaranteed that we will converge to the unique optimum.
	
	\begin{appendixpropn}
		\label{appendixpropn:Total Projections Batch version converges to w*}
		Let $w^*$ be the minimizer of $G(w) = \dfrac{1}{|\S|}\slsins d_s \left[\dfrac{|\phi_s^\tr w - V_s|^2}{2||\phi_s||_2^2}\right]$. Then if the sequence $\{w_1,w_2,w_3,\dots \}$ is obtained by successive total projection operations, starting from some arbitrary point $w_0 \in \R^n$, then $w_\infty = w^*$
	\end{appendixpropn}
	\begin{proof}
		
		Consider the algorithm $w_{k+1} = w_k - \alpha_k (TP(w_k))$ where $TP(w_k) = \nabla_{w} G(w_k) = \dfrac{1}{|\S|}\slsins d_s \left[\dfrac{(\phi_s^\tr w_k - V_s)\phi_s}{||\phi_s||_2^2}\right]$ and $\alpha_k$ is some step size sequence. This is a gradient descent algorithm on $G(\cdot)$. 
		It has been proved in literature \cite{BoydCVO} that a (batch) gradient descent algorithm converges to the local minimizer. Since we have shown that $G(\cdot)$ is a convex function over a convex set, it has a single local minimizer, which is also the global optimum.

		We start with the second order Taylor series expansion of $G(\cdot)$ at some point $y \in \R^n$ in the neighborhood of $w\in \R^n$, and some z between $w$ and $y$, we have
		
		\begin{align*}
			G(y) &= G(w) + \nabla G(w)^\tr (y-w) + \frac{1}{2} (y-w)^\tr Hess(G(z))(y-w)\\
			\intertext{By \hyperref[appendixpropn: Hessian of G is bounded above]{proposition \ref{appendixpropn: Hessian of G is bounded above}}, $H(G(\cdot))$ is bounded above by $1$}
			G(y) &\leq G(w) + \nabla G(w)^\tr (y-w) + \frac{1}{2} ||y-w||_2^2\\
			\intertext{In gradient descent, we proceed in the opposite direction of the gradient. $\therefore y-w = -\alpha \nabla G(w)$ }
			G(y) &\leq G(w) -\alpha \nabla G(w)^\tr \nabla G(w) + \frac{1}{2} ||y-w||_2^2\\
			\intertext{Then, $||y-w||_2 = \alpha ||\nabla G(w)||_2$}
			G(y) &\leq G(w) -\alpha || \nabla G(w)||_2^2 + \frac{1}{2} (\alpha \nabla ||G(w)||_2)^2\\
			G(y) &\leq G(w) -(\alpha-\frac{\alpha^2}{2}) || \nabla G(w)||_2^2 
			\intertext{We want $\alpha-\dfrac{\alpha^2}{2} = \alpha\left(1-\dfrac{\alpha}{2}\right) > 0$\label{appendixeqn: alpha between 0 2/m}. Setting $\alpha-\dfrac{\alpha^2}{2} =c >0$ we get $\alpha \in \left(0,2\right)$:}
			G(y) &\leq G(w) - c || \nabla G(w)||_2^2 \tag{for some constant $c > 0$}\\
			G(y) &< G(w)
			\intertext{Since $TP(w) =\nabla G(w)$, we have $y = w - \alpha TP(w)$ where $\alpha \in(0,2)$}
			G(w - \alpha TP(w)) &< G(w)\\
			\intertext{If we label the successive iterates as $w_k$ and $w_{k+1}$, and the step size for the k'th step as $\alpha_k$:}
			G(w_{k+1}) &< G(w_k)\tag{for $\alpha_k \in \left(0,2\right)$}\\
			\intertext{Let $w^* = \argmin\limits_{w\in\R^n} G(w)$. Then:}
			G(w_{k+1}) - G(w^*) &< G(w_k) - G(w^*)\\
			\intertext{Then for some constant $\gamma_k<1, \gamma_{k}\in \R$:}
			G(w_{k+1}) - G(w^*) &= \gamma_k(G(w_k) - G(w^*))
			\intertext{Similarly, for some constant $\gamma_{k-1}<1, \gamma_{k-1}\in \R$:}
			G(w_{k+1}) - G(w^*) &= \gamma_{k-1}\gamma_k(G(w_{k-1}) - G(w^*))\\
			&= \dots\\
			G(w_{k+1}) - G(w^*) &= \left(\prod\limits_{i=0}^{k} \gamma_i \right)(G(w_0) - G(w^*))\tag{where $\gamma_i<1\quad \forall i \in \{0,\dots,k\}$}\\
			\intertext{Now we take the limit as $k\to \infty$}
			\lim\limits_{k\to\infty}\left(G(w_{k+1}) - G(w^*) \right)&= \lim\limits_{k\to\infty}\left( \left(\prod\limits_{i=0}^{k} \gamma_i \right)(G(w_0) - G(w^*))\right)\tag{where $\gamma_i<1\quad \forall i \in  \lim\limits_{k\to\infty} \{0,\dots,k\}$}\\
			G(w_\infty) - G(w^*)&= \left(\lim\limits_{k\to\infty} \prod\limits_{i=1}^{k} \gamma_i \right)(G(w_1) - G(w^*))\tag{where $\gamma_i<1\quad \forall i \in  \lim\limits_{k\to\infty} \{1,\dots,k\}$}\\
			\intertext{Since the product of infinite numbers less than 1 is 0, we have:}
			G(w_\infty) - G(w^*)&=0\\
			G(w_\infty) &=G(w^*)\\
			\intertext{Since $G(\cdot)$ is convex over $\R^n$, there $w^*$ is the unique minimizer}
			w_\infty &=w^*
		\end{align*}
	\end{proof}
	Thus we show convergence. To get rate of convergence, we need to make some assumptions about $\alpha$.
	\subsection{Conditions on the step size of the total projection algorithm}
	We showed in \hyperref[appendixpropn:Total Projections Batch version converges to w*]{proposition \ref{appendixpropn:Total Projections Batch version converges to w*}} that the batch version of Total Projections converges to the global optimum for $\alpha_k \in (0,2)$. Now we will study what is the ideal step size to take in this above range as part of the TP algorithm.
	
	\begin{appendixpropn}
		\label{appendixpropn:alpha is 1}
		The optimal step-size $\alpha_{OPT} = 1$
	\end{appendixpropn}
	\begin{proof}
		We have already seen in \hyperref[appendixpropn:Total Projections Batch version converges to w*]{proposition \ref{appendixpropn:Total Projections Batch version converges to w*}} that for some $w_{k+1}$ in the neighborhood of $w_k$, we have
		\begin{equation}
			\label{appendixeqn: wk+1 wk inequality wrt alpha}
			G(w_{k+1}) \leq G(w_k) -(\alpha-\frac{\alpha^2}{2}) || \nabla G(w_k)||_2^2
		\end{equation}
		which is quadratic in $\alpha$. If we want to minimize the LHS, with respect to $\alpha$, we set the derivative of the RHS with respect to $\alpha$ to 0. Thus for an optimal alpha, viz. $\alpha_{OPT}$ we have:
		\begin{align}
			\nabla_\alpha G(w_{k+1}) &=0\nonumber\\
			\nabla_\alpha \left(G(w_k) -\left[\alpha_{OPT}-\frac{\alpha_{OPT}^2}{2}\right] || \nabla G(w_k)||_2^2 \right)&=0\nonumber\\
			\nabla_\alpha G(w_k) - \nabla_\alpha\left[\alpha_{OPT}-\frac{\alpha_{OPT}^2}{2}\right] || \nabla G(w_k)||_2^2 = 0\nonumber\\
			\intertext{Since $\nabla_\alpha G(w_k) = 0$ and $\nabla G(w_k)||_2^2$ is independent of $\alpha$}
			\nabla_\alpha\left[\alpha_{OPT}-\frac{\alpha_{OPT}^2}{2}\right] =0\nonumber\\
			\intertext{which leads to:}
			1 - \alpha_{OPT} &= 0\nonumber
			\intertext{Thus}
			\alpha_{OPT} = 1
		\end{align}
	\end{proof}
	In light of this, the stochastic update equation for the batch version of the TP algorithm is $$w_{k+1} = w_k-\left(\frac{1}{|\S|}\slsins d_s \left[ \dfrac{\phi_s^\tr w - V_s}{||\phi_s||^2}\right]\phi_s \right)$$

	\subsection{Convergence Rate of the Total Projections Algorithm}
	Now we are ready to show the exponential convergence rate for the Total Projections algorithm. We will now show the rate of convergence of the TP algorithm is exponential when $\Phi$ has full column rank using:
	\begin{enumerate}[itemsep=-0.05in,leftmargin=1in,label=(\alph*)]
		\item $G(w_{k+1}) - G(w^*)\leq G(w_k)- G(w^*) - \frac{1}{2} \left||\nabla G(w_k)|\right|_2^2$ 
		
		\item $||\nabla G(w_k)||_2^2  \geq \dfrac{2(G(w_k) -G(w^*))}{2-\lambda_{min}}$
		
	\end{enumerate}

	\begin{appendixpropn}
		\label{appendixpropn:TP has exponential rate of convergence}
		Rate of convergence of the TP algorithm is exponential when $\Phi$ has full column rank
	\end{appendixpropn}
	\begin{proof}
		
		Firstly, from \hyperref[appendixpropn: G wk+1 G wk inequality after substituting alpha]{proposition \ref{appendixpropn: G wk+1 G wk inequality after substituting alpha}}, we have:
		\begin{align}
			G(w_{k+1}) - G(w^*)&\leq G(w_k)- G(w^*) - \frac{1}{2} \left||\nabla G(w_k)|\right|_2^2\nonumber
			\intertext{Then from \hyperref[appendixpropn: Bound of norm square of grad G(wk)]{proposition \ref{appendixpropn: Bound of norm square of grad G(wk)}} we have:}
			||\nabla G(w_k)||_2^2  &\geq \frac{2(G(w_k) -G(w^*))}{2-\lambda_{min}}\tag{Note: $\lambda_{min}\leq\lambda_{max}\leq 1$}\nonumber
			\intertext{Combining, we get:}
			G(w_{k+1}) - G(w^*)&\leq G(w_k)- G(w^*) - \frac{1}{2}\frac{2(G(w_k) -G(w^*))}{2-\lambda_{min}}\nonumber\\
			&= [G(w_k)- G(w^*)]\left[1- \frac{1}{2-\lambda_{min}}\right]\nonumber\\
			&=[G(w_k)- G(w^*)] \left[\frac{1-\lambda_{min}}{2-\lambda_{min}}\right]\nonumber
			\intertext{We now can create a telescoping product. For successive iterates $\{w_1,\dots,w_k\}$:}
			G(w_{k+1}) - G(w^*)&\leq [G(w_k)- G(w^*)] \left[\frac{1-\lambda_{min}}{2-\lambda_{min}}\right]\nonumber\\
			&\leq [G(w_{k-1})- G(w^*)] \left[\frac{1-\lambda_{min}}{2-\lambda_{min}}\right]^2\nonumber\\
			&\leq \dots\nonumber\\
			&\leq [G(w_{0})- G(w^*)] \left[\frac{1-\lambda_{min}}{2-\lambda_{min}}\right]^{n+1}
		\end{align}
		
		Thus we have a Q-linear rate of convergence, also known as exponential rate of convergence
	\end{proof}
	
	\begin{appendixpropn}
		\label{appendixpropn: G wk+1 G wk inequality after substituting alpha}
		$G(w_{k+1}) - G(w^*)\leq G(w_k)- G(w^*) - \frac{1}{2} \left||\nabla G(w_k)|\right|_2^2$
	\end{appendixpropn}
	\begin{proof}
		From \hyperref[appendixeqn: wk+1 wk inequality wrt alpha]{equation \ref{appendixeqn: wk+1 wk inequality wrt alpha}} in \hyperref[appendixpropn:alpha is 1]{proposition \ref{appendixpropn:alpha is 1}}, we can see
		$G(w_{k+1}) \leq G(w_k) -(\alpha-\frac{\alpha^2}{2}) || \nabla G(w_k)||_2^2$. Substituting $\alpha = 1$ from \hyperref[appendixpropn:alpha is 1]{proposition \ref{appendixpropn:alpha is 1}}, we get 
		\begin{align*}
			G(w_{k+1}) &\leq G(w_k) -\frac{1}{2} || \nabla G(w_k)||_2^2\\
			\intertext{Then subtracting $G(w^*)$ from both sides:}
			G(w_{k+1}) - G(w^*)&\leq G(w_k)- G(w^*) - \frac{1}{2} \left||\nabla G(w_k)|\right|_2^2
		\end{align*}
	\end{proof}
	
	\begin{appendixpropn}
		\label{appendixpropn: Bound of norm square of grad G(wk)}
		$||\nabla G(w_k)||_2^2  \geq \dfrac{2}{2-\lambda_{min}}(G(w_k) -G(w^*))$
	\end{appendixpropn}
	\begin{proof}
		From \hyperref[appendixpropn: G is mu strongly convex]{proposition \ref{appendixpropn: G is mu strongly convex}} we note that when $\Phi$ has full column rank, then $G(\cdot)$ is $\mu-$ strongly convex, with $\mu = \lambda_{min}$, where $\lambda_{min}$ is the least eigenvalue of $Hess(G(\cdot))$
		
		Let $w_{k+1}\in \R^n$ be some point in the neighborhood of $w_k\in\R^n$, and z be a point in the interval $[w_k,w_{k+1}]$. Then by second order Taylor series expansion,
		\begin{align*}
			G(w_{k+1}) &= G(w_k) +\nabla_w G(w_k)^\tr(w_{k+1}-w) + \frac{1}{2} (w_{k+1}-w_k)^\tr Hess(G(z))(w_{k+1}-w_k)
			\intertext{Since the Hessian is bounded below:}
			G(w_{k+1}) &\geq G(w_k) +\nabla_w G(w_k)^\tr(w_{k+1}-w_k) + \frac{1}{2} (w_{k+1}-w_k)^\tr \lambda_{min}(w_{k+1}-w_k)\\
			&= G(w_k) +\nabla_w G(w_k)^\tr(w_{k+1}-w_k) + \frac{\lambda_{min}}{2} ||w_{k+1}-w_k||_2^2
		\end{align*}
		But if $w_{k+1} = w_k - \nabla G(w_k)$ or $w_{k+1}-w_k = - \nabla G(w_k)$. Thus,
		\begin{align*}
			G(w_{k+1}) &\geq G(w_k) +\nabla_w G(w_k)^\tr(-\nabla G(w_k)) + \frac{\lambda_{min}}{2} \left|\left|- \nabla G(w_k)\right|\right|_2^2\\
			&= G(w_k) -||\nabla_w G(w_k)||_2^2 + \frac{\lambda_{min}}{2} ||\nabla G(w_k)||_2^2\\
			&= G(w_k) -\left[\frac{2-\lambda_{min}}{2}\right]||\nabla_w G(w_k)||_2^2 \\
			\left[\frac{2-\lambda_{min}}{2}\right]||\nabla G(w_k)||_2^2  &\geq G(w_k) -G(w_{k+1}) \\
			\intertext{But $G(w_k) -G(w_{k+1}) > G(w_k) -G(w^*)$. Thus }
			\left[\frac{2-\lambda_{min}}{2}\right]||\nabla G(w_k)||_2^2  &\geq G(w_k) -G(w^*)\\
			||\nabla G(w_k)||_2^2  &\geq \left[\frac{2}{2-\lambda_{min}}\right](G(w_k) -G(w^*))\tag{Note: $\lambda_{min}\leq\lambda_{max}\leq 1$}\\
		\end{align*}
	\end{proof}

	%\section{Convergence of the Scale Invariant Monte Carlo without momentum}
	\section{CONVERGENCE OF SCALE INVARIANT MONTE CARLO WITHOUT MOMENTUM}
	\label{appendixSection:Convergence of Scale Invariant monte carlo without momentum}
	
	In this section, we will show the convergence point of the Scale Invariant Monte Carlo.
	
	\subsection{Notation and Problem Setup}
	Firstly note that under the linear function approximation regime, we are solving the overdetermined system $\Phi w = \widetilde{V}$ with $|\S| = m$ hyperplanes of the form $\H_s\equiv \phi_s^\tr w - \widetilde{V}_s = 0$. We know that the first visit Monte Carlo, is an unbiased estimator of the value function $V$ for each state. Thus $\widetilde{V}$ is an unbiased estimator of $V$
	
	The sampling of the hyperplanes is as per the stationary distribution of the transition matrix $\P$. The stationary distribution is given by $\pi$ with the probability of $\H_s$ given by $\pi_s$. We denote the diagonal matrix associated with $\pi$ as $D$. Finally, we define a normalization matrix $N$ where N is a diagonal matrix with $N_{(i,i)} = \frac{1}{||\phi(i)||_2}$

	Let us define $TP_k(\cdot):\R^n\mapsto\R^n$ to be the function that takes a point $w_k$ and gives us the shift in $w_k$ for the k'th iteration. 
	
	Thus $TP_k(w_k) = \dfrac{1}{\T}\sum\limits_{i=1}^\T \dfrac{\phi_i^\tr w_k - \widetilde{V}_i}{||\phi_i||^2}\phi_i$. Further, let us define $TP(\cdot)$ \hyperref[appendixSection:TP Algorithm non-stochastic properties]{(ref. section \ref{appendixSection:TP Algorithm non-stochastic properties})} as $TP(w_k) = \sum\limits_{s\in\S} \pi_s \dfrac{\phi_s^\tr w_k - \widetilde{V}_s}{||\phi_s||^2}\phi_s$
	
	$TP_k(w_k)$ then depends on the trajectory for the Monte Carlo. In other words, it depends on the set of hyperplanes sampled (which is random), where the number of hyperplanes sampled is also random. 
	We will assume that the stopping time $\tau$ is obtained by some independent random process \cite{Cinlar2011}. In other words, $\tau$ is an independent Random Variable. We make this assumption as if $\tau$ is dependent explicitly on landing at certain states in the Markov Chain, then we lose the stationarity of the distribution as all states will eventually reach the absorbing states.

	The limiting ODE that the stochastic update equation $w_{k+1} = w_k - \alpha_k TP_k(w_k)$ tracks is given by $\dot{w}(t) = h_{k+1}(w(t))$ where $h_{k+1}(w) = \E \left[TP_k(w_k)\big| \F_k\right]$ for the filtration $\F_k = \{w_0,\dots,w_k\}$.	
	Note that $\dot{w}(t) = h_{k+1}(w(t))$ is a well studied o.d.e which converges to the point where $\dot{w}(t) = 0$ \cite{Borkar}. Let us denote this point as $w^*$. Then the problem in this section is to find the point of convergence, $w^*$.
	
	\subsection{Putting the update equation in standard form:}
	\label{appendixSection: Update equation without momentum Standard Form}
	Consider the update equation $w_{k+1} = w_k - \alpha_k TP_k(w_k)$. Given the filtration $\F_k = \{w_0,\dots,w_k\}$, we wish to find $h_{k+1}(w)$. Let $\{1\dots\T\}$ be the set of unique hyperplanes sampled on the k'th run of trajectory. Then:
	
	\begin{equation*}
		h_{k+1}(w) = \E \left[TP_k(w_k)\big| \F_k\right] = \E \left[\frac{1}{\T}\sum\limits_{i=1}^\T \dfrac{\phi_i^\tr w_k - \widetilde{V}_i}{||\phi_i||^2}\phi_i\bigg|\F_k\right]
	\end{equation*}

	We note that $\T$, the set of hyperplanes $\{1\dots\T\}$ sampled, as well as $\widetilde{V}$ are all random variables. To simplify from the three random variables, first we write the above expression as an expectation over the conditional expectation given $\T$. Then:
	
	\begin{align}
		h_{k+1}(w) &= \E_{\T} \left[\E \left[\frac{1}{\T}\sum\limits_{i=1}^\T \dfrac{\phi_i^\tr w_k - \widetilde{V}_i}{||\phi_i||^2}\phi_i\bigg|\F_k,\T\right]\right]\nonumber
		\intertext{By linearity of expectation, we can take the expectation inside the brackets:}	
		&= \E_{\T} \left[ \frac{1}{\T}\sum\limits_{i=1}^\T \E \left[\dfrac{\phi_i^\tr w_k - \widetilde{V}_i}{||\phi_i||^2}\phi_i\bigg|\F_k,\T\right]\right]\label{appendixeqn:Double expectation with tau}
		\intertext{But any hyperplane $\H_i\equiv \phi_i^\tr w - V_i$ is chosen with probability equal to $\pi_i$ where $\pi$ is the stationary distribution. Thus the weights $d_s$ that we used in \hyperref[appendixSection:TP Algorithm non-stochastic properties]{Section \ref{appendixSection:TP Algorithm non-stochastic properties}} now take the form $d_s = \pi_s |\S|$. 		
			Thus $\forall i \in \{1,\dots,\T\}$, :}
		\E_{(i,\widetilde{V})} \left[\dfrac{\phi_i^\tr w_k - \widetilde{V}_i}{||\phi_i||^2}\phi_i\bigg| \F_k,\T\right] &= \dfrac{1}{|\S|}\slsins \pi_s |\S|\cdot  \E_{\widetilde{V}} \left[\dfrac{\phi_s^\tr w_k - \widetilde{V}_s}{||\phi_s||^2}\phi_s\right] \\
		&= \slsins \pi_s \left[\dfrac{\phi_s^\tr w_k - V_s}{||\phi_s||^2}\phi_s\right]\nonumber
		\intertext{Substuting this back in \eqref{appendixeqn:Double expectation with tau}, we get:}
		h_{k+1}(w) &=\E_{\T} \left[ \frac{1}{\T} \sum\limits_{i=1}^\T \left[\slsins \pi_s\dfrac{\phi_s^\tr w_k - V_s}{||\phi_s||^2}\phi_s\right] \right]\nonumber
		\intertext{Since each of the terms in the sum is the same:}
		&=\E_{\T} \left[  \slsins \pi_s\dfrac{\phi_s^\tr w_k - V_s}{||\phi_s||^2}\phi_s \right]\nonumber
		\intertext{Since each term inside is independent of $\tau$, the expectation stays the same. Thus:}
		h_{k+1}(w) &= \slsins \pi_s\dfrac{\phi_s^\tr w_k - V_s}{||\phi_s||^2}\phi_s \nonumber
		\intertext{But the RHS is simply $TP(w_k)$ for the Monte Carlo. Thus:}
		h_{k+1}(w) &= TP(w_k)
	\end{align}
	
	Since the function $h_k(\cdot)$ is constant for all $k$, i.e. $h_k(\cdot) = TP(\cdot)$, we can simply refer to this as $h(\cdot) = TP(\cdot)$.
	
	Now we are in a position to put our update equation in standard form. Let $M_{k+1} = TP_k(w_k) - \E[TP_k(w_k)\big|\F_k]$, then we can write the update rule as:
	
	\begin{align}
		w_{k+1} = w_k - \alpha_k(h(w_k)+M_{k+1})
	\end{align}
	where $\alpha_{k}$ is as defined in \hyperref[section:StepSizeSeq]{section \ref{section:StepSizeSeq}}. Further, $h(w_k) = TP(w_k)$ and $M_{k+1} = TP_k(w_k) - TP(w_k)$
	
	In the next section we will show that the four conditions required for convergence (\cite{Borkar}) are satisfied. In the section after that we will show the point it converges to.
	
	\subsection{Showing satisfaction of assumptions A1-A4 required for convergence}
	
	We need to show the following assumptions are satisfied:
	\begin{enumerate}
		\item The map $h(\cdot)$ is Lipschitz
		\item Step sizes $\{\alpha_k\}$ are positive scalars satisfying $\slktoinf \alpha_k = \infty$ and $\slktoinf \alpha_k^2 < \infty$
		\item $\{M_k\}$ is a martingale difference sequence with respect to the filtrations $\F_k$. 
		
		Further $\{M_k\}$ are square integrable with $\E \left[||M_{k+1}||^2\bigg|F_k\right] \leq K(1+||w_k||^2)$ a.s. for some positive constant K
		\item The iterates $\{w_k\}$ remain bounded almost surely		
	\end{enumerate}
	
	We will show these in order.
	
	\subsubsection[The error function with weights as per stationary distribution is Lipshitz]{The map $h(\cdot)$ is Lipschitz}
	\label{appendixSection: h is Lipschitz}
	In \hyperref[appendixSection:gradient of the error function is Lipschitz]{Appendix section \ref{appendixSection:gradient of the error function is Lipschitz}}, we showed that the TP update is Lipschitz for general weights $d_s$ as long as $\slsins d_s \leq |\S|$. Now we are considering the specific case where $d_s = \pi_s |\S|$. Since $\pi$ is a probability distribution (and therefore sums to 1), we satisfy $\slsins d_s \leq |\S|$. Thus $h(\cdot)$ is Lipschitz.
	
	\subsubsection[The step size sequence is not summable, but square summable]{The sequence $\{\alpha_k\}$ is square summable but not summable}
	\label{appendixSection: alpha without momentum is square summable not summable}
	\begin{appendixpropn}
		\label{appendixpropn:single iterate has proper step size sequence}
		The step size sequence $\{\alpha_k\}_{k=1}^\infty$ satisfies $\slktoinf \alpha_k = \infty$ and $\slktoinf\alpha_k^2 < \infty$
	\end{appendixpropn}
	\begin{proof}\vspace{-0.1in}
		We provide the full proof for \hyperref[propn:MonteCarlo A2 step size sequence satisfies properties]{proposition \ref{propn:MonteCarlo A2 step size sequence satisfies properties}} as  follows.
		\begin{align*}
			\slktoinf \alpha_k &= \slktoinf \dfrac{\theta_k \eta_k}{||TP_k(w_k)||}\\ 
			&= \slktoinf \dfrac{\theta_k}{k^p\cdot||TP_k(w_k)||}\\
			\intertext{\vspace{-0.1in}Expanding $\theta_k = \dfrac{||TP_k(w_k)||^2}{||\Delta TP_k(w_k)||}$, we get:}
			&= \slktoinf \dfrac{||TP_k(w_k)||}{k^p\cdot||\Delta TP_k(w_k)||}
			\intertext{\vspace{-0.1in}let $\vartheta_k = \dfrac{||TP_k(w_k)||}{||\Delta TP_k(w_k)||}$. Then :}
			&= \slktoinf \dfrac{\vartheta_k}{k^p}
		\end{align*}
		
		\vspace{-0.1in}
		We first show the almost sure lower bounds on $TP_k(w_k)$ and $\Delta TP_k(w_k)$. Note that  $TP_k(w_k) = \frac{1}{\T}\sum\limits_{i=1}^\T \left[\frac{\phi_i^\tr w_k - \widetilde{V}_i}{||\phi_i||^2}\right]\phi_i = \frac{1}{\T}\sum\limits_{i=1}^\T \left[\frac{\phi_i\phi_i^\tr}{||\phi_i||^2}\right]w_k - \frac{1}{\T}\sum\limits_{i=1}^\T \left[\frac{ \widetilde{V}_i\phi_i}{||\phi_i||^2}\right]$ is almost surely not equal to 0 for random $w_k\in \R^n$. For $\Delta TP_k(w_k)$, we write:
		\begin{align*}
			\Delta TP_k(w_k) &= TP_k(w_k-TP_k(w_k))-TP_k(w_k)\\
			&= \frac{1}{\T}\sum\limits_{i=1}^\T \left[\frac{\phi_i^\tr (w_k-TP_k(w_k)) - \widetilde{V}_i}{||\phi_i||^2}\right]\phi_i - \frac{1}{\T}\sum\limits_{i=1}^\T \left[\frac{\phi_i^\tr w_k - \widetilde{V}_i}{||\phi_i||^2}\right]\phi_i\\
			&= -\frac{1}{\T}\left[\sum\limits_{i=1}^\T \dfrac{\phi_i\phi_i^\tr}{||\phi_i||^2}\right] TP_k(w_k)
		\end{align*}
		
		We firstly note that $TP_k(w_k)\neq 0$ almost surely as the iterate doesn't lie on the hyperplanes that $TP_k(\cdot)$ uses. (WLOG, if we do lie on the intersection of the hyperplanes, then we may choose other hyperplanes). Further, for any given vector $TP_k(w_k)$, the chance of $TP_k(w_k)$ being perpendicular to all the vectors$\{\phi_i\}_{i=1}^\T$ is almost surely 0. Thus $\Delta TP_k(w_k) \neq 0\enspace a.s.$. 
		
		For the upper bounds, we first note that the iterates $\{w_k\}$ are bounded a.s. as per \hyperref[propn:MonteCarloA4]{proposition \ref{propn:MonteCarloA4}} and \hyperref[appendixsubsection: A4 Stability Criterion]{Appendix \ref{appendixsubsection: A4 Stability Criterion}}. Then we further have that the estimates $\widetilde{V}$ are bounded by $R_{max}/1-\gamma$ where $R_{max}$ is the maximum reward on transitions and $\gamma$ is the discounting factor. Since, the iterates $\{w_k\}$ are bounded, $TP_k(w_k)$ and $\Delta TP_k(w_k)$ are upper bounded.

		Now by these statements, $\vartheta_k = \dfrac{||TP_k(w_k)||}{||\Delta TP_k(w_k)||}$ is upper and lower bounded almost surely.	\footnote{Note: In our simulations, such points where $||\Delta TP_k(w_k)|| \sim 0$ were never reached and iterates were stable even very close to the solution (see \hyperref[fig:Evidence for Curvature based Step size]{Figure \ref{fig:Evidence for Curvature based Step size}}). But to ensure algorithmic stability (given limited floating point precision), we can physically set the updates to not occur when $\Delta TP_k(w_k)$ is below a certain $\varepsilon$ (say $10^{-6}$) threshold.}
		
		Then, let $\overline{\vartheta} = \sup\limits_k \theta_k$ and $\underline{\vartheta} = \inf\limits_k \vartheta_k$. Then 
		\begin{align*}
			\slktoinf \alpha_k &= \sltoinf \dfrac{\vartheta_k}{k^p} \\
			&\geq \slktoinf \dfrac{\underline{\vartheta}}{k^p} \\
			&= \underline{\vartheta} \slktoinf \dfrac{1}{k^p} \\
			&= \underline{\vartheta}\times \infty \\
			&= \infty
		\end{align*}
		
		Similarly, 
		\begin{align*}
			\slktoinf \alpha_k^2 &=\slktoinf \vartheta_k^2 \eta_k^2 \\
			&\leq \overline{\vartheta} \slktoinf \eta_k^2 \\
			&= \overline{\vartheta} \slktoinf \frac{1}{k^{2p}}\\
			\intertext{Now since $\slktoinf \dfrac{1}{k^{2p}}$ is finite, and $\overline{\vartheta}$ is finite. Thus:}
			\slktoinf \alpha_k^2 &< \infty
		\end{align*}
	\end{proof}
	
	\subsubsection[Mk is a martingale difference sequence that is square integrable]{$\{M_k\}$ is a martingale difference sequence that is square integrable:}
	\label{appendixSection: Margingale difference sequence square integrable}
	We need to show that $\E[M_{k+1}\big|\F_k] = 0\enspace a.s.$ and $\E[|M_{k+1}||^2\big|\F_k ]\leq K(1+||w_k||^2)$ where $M_{k+1} = TP_k(w_k) - TP(w_k)$
	
	For the first part, we have that 
	\begin{align*}
		\E[M_{k+1}\big|\F_k] &= \E\left[TP_k(w_k) - TP(w_k)\big|\F_k\right]\\
		&= \E\left[TP_k(w_k)\big|\F_k\right] - \E\left[TP(w_k)\big|\F_k\right]\\
		\intertext{Then since $\E[TP(w_k)|\F_k] = TP(w_k)$, we get:}
		&= \E\left[TP_k(w_k)\big|\F_k\right] - TP(w_k)
		\intertext{But we already computed in \hyperref[appendixSection: Update equation without momentum Standard Form]{appendix \ref{appendixSection: Update equation without momentum Standard Form}} that $\E\left[TP_k(w_k)\big|\F_k\right] = TP(w_k)$. Thus:}
		&= TP(w_k) - TP(w_k)\\
		&=0
	\end{align*}
	
	For the second part, we write 
	\begin{align*}
		M_{k+1} &= TP_k(w_k) - TP(w_k)\\
		&= \frac{1}{\T}\sum\limits_{i=1}^\T \left[\frac{\phi_i^\tr w_k - \widetilde{V}_i}{||\phi_i||^2}\right]\phi_i - \slsins \pi_s \left[\frac{\phi_s^\tr w_k - V_s}{||\phi_s||^2}\right]\phi_s\\
		&= \left[\frac{1}{\T}\sum\limits_{i=1}^\T \frac{\phi_i\phi_i^\tr}{||\phi_i||^2} - \slsins\pi_s\frac{\phi_s\phi_s^\tr}{||\phi_s||^2}   \right]w_k - \left[\frac{1}{\T}\sum\limits_{i=1}^\T \frac{\widetilde{V}_i\phi_i}{||\phi_i||^2} - \slsins\pi_s\frac{V_s\phi_s}{||\phi_s||^2}   \right]
		\intertext{We can call $\left[\frac{1}{\T}\sum\limits_{i=1}^\T \frac{\phi_i\phi_i^\tr}{||\phi_i||^2} - \slsins\pi_s\frac{\phi_s\phi_s^\tr}{||\phi_s||^2}   \right]$ as A and $\left[\frac{1}{\T}\sum\limits_{i=1}^\T \frac{\widetilde{V}_i\phi_i}{||\phi_i||^2} - \slsins\pi_s\frac{V_s\phi_s}{||\phi_s||^2}\right]$ as b. Then:}
		M_{k+1} &= Aw_k - b
	\end{align*}
	Note that the eigenvalues of $A$ are bounded as each term $\dfrac{\phi_i\phi_i^\tr}{||\phi_i||^2}$ has a maximum eigenvalue of 1 as per \hyperref[appendixpropn: max eigenval of phiiphii.tr by norm phi sq is 1]{proposition \ref{appendixpropn: max eigenval of phiiphii.tr by norm phi sq is 1}}.  Similarly $b$ is bounded as $\widetilde{V}$ is bounded by $\dfrac{R_{max}}{1-\gamma}$ where $R_{max}$ is the maximum reward on transitions between states and $\gamma$ is the discounting factor.

	Now we see that $M_{k+1} = Aw_k - b$ is linear in $w_k$ with bounded coefficients. Thus $\E \left[||Aw_k - b||^2\big|\F_k\right]$ is quadratic in $w_k$. Now it's straightforward to see that there exists some constant $K$ such that $\E \left[||M_{k+1}||^2\big|\F_k\right] \leq K(1+||w_k||^2)$.

	\subsubsection{The iterates remain bounded almost surely}
	\label{appendixsection: iterates remain bounded almost surely no momentum}
	We have already shown this in \hyperref[propn:MonteCarloA4]{proposition \ref{propn:MonteCarloA4}}. We also provide a proof based on stability criterion from \cite{LakshmiNarayanBhatnagar} in \hyperref[appendixsubsection: A4 Stability Criterion]{appendix section \ref{appendixsubsection: A4 Stability Criterion}}
	
	Now that we satisfy conditions A1-A4  for iterate convergence \cite{Borkar} in sections \hyperref[appendixSection: h is Lipschitz]{\ref{appendixSection: h is Lipschitz}} to \hyperref[appendixsection: iterates remain bounded almost surely no momentum]{\ref{appendixsection: iterates remain bounded almost surely no momentum}},
	we know that the iterates will converge. It remains to be seen where it converges to, which we will cover in the next section.
	
	\subsection{Convergence point of the Scale Invariant Monte Carlo}
	\label{appendixsubsection:Convergence point of the Scale Invariant Monte Carlo}
	In this section we will show that:
	\begin{enumerate}[itemsep=0.05in,leftmargin=1in,label=(\alph*)]
		\item If $w^*$ is the point of convergence of the Scale Invariant Monte Carlo Algorithm, then \\
		$w^* = \left[\left(\Phi^\tr N D N \Phi\right)^{-1} \Phi^\tr N D N\right] V$
		
		using
		\item $\left(\slsins\pi_s \left[\dfrac{\phi_s\phi_s^\tr  }{||\phi_s||_2^2}\right]\right)w^*=\slsins\pi_s \left[\dfrac{\phi_s V_s}{||\phi_s||_2^2}\right]$
		
		\item $\slsins\pi_s \left[\dfrac{\phi_s\phi_s^\tr  }{||\phi_s||_2^2}\right] = \Phi^\tr N D N \Phi$
		\item $\slsins\pi_s \left[\dfrac{\phi_s V_s}{||\phi_s||_2^2}\right] = \Phi^\tr N D N V$
		
	\end{enumerate}
	
	\begin{appendixpropn}
		\label{appendixpropn: Derivation of w^*}
		The convergence point $w^*$ of our algorithm, which is the stable point of the o.d.e that our stochastic update equation tracks, satisfies the condition
		$$w^* = \left[\left(\Phi^\tr N D N \Phi\right)^{-1} \Phi^\tr N D N\right] V$$
	\end{appendixpropn}
	\begin{proof}
		Since we are looking for the point $w^*$ where $\expectation TP(w^*)=0$, from \hyperref[appendixpropn: Condition that stable point w^* satisfies]{proposition \ref{appendixpropn: Condition that stable point w^* satisfies}} we have:
		\begin{align}
			\left(\slsins\pi_s \left[\dfrac{\phi_s\phi_s^\tr  }{||\phi_s||_2^2}\right]\right)w^* &= \slsins\pi_s \left[\dfrac{\phi_s V_s}{||\phi_s||_2^2}\right]\nonumber
			\intertext{From \hyperref[appendixpropn: sum Pi phi V equals Phi tr N D N V]{proposition \ref{appendixpropn: sum Pi phi V equals Phi tr N D N V}}, we have that  $\slsins\pi_s \left[\dfrac{\phi_s V_s}{||\phi_s||_2^2}\right] = \Phi^\tr N D N V$. Thus:}
			\left(\slsins\pi_s \left[\dfrac{\phi_s\phi_s^\tr  }{||\phi_s||_2^2}\right]\right) w^* &= \Phi^\tr N D N V\nonumber
			\intertext{From \hyperref[appendixpropn: Pi phi phi tr equals phi N D N phi]{proposition \ref{appendixpropn: Pi phi phi tr equals phi N D N phi}}, we have that $\slsins\pi_s \left[\dfrac{\phi_s\phi_s^\tr  }{||\phi_s||_2^2}\right] = \Phi^\tr N D N \Phi$. Thus:}
			\Phi^\tr N D N \Phi w^* &= \Phi^\tr N D N V\nonumber
			\intertext{Multiplying by $\left[\Phi^\tr N D N \Phi\right]^{-1}$ on both sides:}
			\left(\Phi^\tr N D N \Phi\right)^{-1}\Phi^\tr N D N \Phi w^* &= \left(\Phi^\tr N D N \Phi\right)^{-1} \Phi^\tr N D N V\nonumber
			\intertext{To finally get:}
			w^* &= \left[\left(\Phi^\tr N D N \Phi\right)^{-1} \Phi^\tr N D N\right] V
		\end{align}
	\end{proof}
	
	\begin{appendixpropn}
		\label{appendixpropn: Condition that stable point w^* satisfies}
		The convergence point $w^*$ of our algorithm, which is the stable point of the o.d.e that our stochastic update equation tracks, satisfies the condition
		$$\left(\slsins\pi_s \left[\dfrac{\phi_s\phi_s^\tr  }{||\phi_s||_2^2}\right]\right)w^*=\slsins\pi_s \left[\dfrac{\phi_s V_s}{||\phi_s||_2^2}\right]$$
	\end{appendixpropn}
	\begin{proof}
		We are looking for the point where $h(w(t)) = 0$. In other words, we are looking for a point $w^*$ where $\expectation TP(w^*)=0$. Then we have:
		\begin{align*}
			\slsins \pi_s \frac{\phi_s^\tr w^* - V_s}{||\phi_s||^2}\phi_s &= 0\\
			\intertext{Which we can directly rewrite to:}
			\left(\slsins\pi_s \left[\dfrac{\phi_s\phi_s^\tr  }{||\phi_s||_2^2}\right]\right)w^* &= \slsins\pi_s \left[\dfrac{\phi_s V_s}{||\phi_s||_2^2}\right]
		\end{align*}
	\end{proof}
	\begin{appendixpropn}
		\label{appendixpropn: Pi phi phi tr equals phi N D N phi}
		$\slsins\pi_s \left[\dfrac{\phi_s\phi_s^\tr  }{||\phi_s||_2^2}\right] = \Phi^\tr N D N \Phi$
	\end{appendixpropn}
	\begin{proof}
		Note that the LHS and RHS are both matrices of size $n\times n$. We will show the equality explicitly for each (i,j)'th entry of this matrix.
		
		For the LHS, the entry at position (i,j) is given by $\slsins\pi_s \left[\dfrac{\phi_s(i)\phi_s(j)}{||\phi_s||_2^2}\right]$
		
		For the RHS, first note that $N D N$ is a diagonal matrix of size $|\S|\times |\S|$. The diagonal entries are given by $\left[N D N\right]_{(s,s)} = \dfrac{\pi_s}{||\phi_s||_2^2}$. Then $NDN\Phi$ has $|\S|$ rows of the form $\dfrac{\pi_s}{||\phi_s||_2^2} \phi_s$. Finally, the entry at the (i,j)'th location of $\Phi^\tr N D N \Phi$, which is a $n\times n$ matrix is given by $\slsins \phi_s(i)^\tr \dfrac{\pi_s}{||\phi_s||_2^2} \phi_s (j)$. Note that this can be rewritten as $\slsins\pi_s \left[\dfrac{\phi_s(i)\phi_s(j)}{||\phi_s||_2^2}\right]$, which is the same as the LHS.
	\end{proof}
	\begin{appendixpropn}
		\label{appendixpropn: sum Pi phi V equals Phi tr N D N V}
		$\slsins\pi_s \left[\dfrac{\phi_s V_s}{||\phi_s||_2^2}\right] = \Phi^\tr N D N V$
	\end{appendixpropn}
	\begin{proof}
		In this case we are dealing with a vector in $\R^n$ for both the LHS and the RHS. We will show equality by showing the i'th entry of this vector on both LHS and RHS are the same.
		
		For the LHS, we have a sum of $|\S|$ vectors of the form $\left(\dfrac{\pi_sV_s}{||\phi_s||_2^2}\right)\phi_s$. Then the entry at i'th location is given by $\slsins \left(\dfrac{\pi_sV_s}{||\phi_s||_2^2}\right)\phi_s(i)$
		
		For the RHS, note that $\left[N D N\right]_{(s,s)} = \dfrac{\pi_s}{||\phi_s||_2^2}$ as in \hyperref[appendixpropn: Pi phi phi tr equals phi N D N phi]{proposition \ref{appendixpropn: Pi phi phi tr equals phi N D N phi}}. Then $N D N V$ is a vector of size $|\S|$ where the entry for state s is given as $\left[N D N V\right]_s = \dfrac{\pi_sV_s}{||\phi_s||_2^2}$. Finally, we have that the entry at the i'th row ($i \in \{1,\dots ,n\}$) in $\Phi^\tr N D N V$ is given by $\dfrac{\pi_sV_s}{||\phi_s||_2^2} \phi_s(i)$, which is the same as the LHS
	\end{proof}

	%\section[Choice of momentum multiplier]{Choice of constant momentum multiplier $\beta$}
	\section{CHOICE OF MOMENTUM MULTIPLIER FOR HEAVYBALL MOMENTUM}
	\label{appendixSection: Choice of constant momentum multiplier}
	We plot the mean error with iterations for different $\beta$ values to do a comparison between the various constant values in Figures \ref{fig:Beta Comparison Non Stochastic} and \ref{fig:Beta Comparison Stochastic}. This will enable us to see reasons for our choice of $\beta=0.5$. 
	
	Note that when we increase $\beta$ beyond 0.5, we see non-smoothness in convergence of the stochastic case. Thus we do not go for $\beta>0.5$ even though it sometimes leads to faster convergence.
	\begin{figure}[htb]
		\begin{subfigure}{.5\linewidth}
			\includegraphics[scale=1]{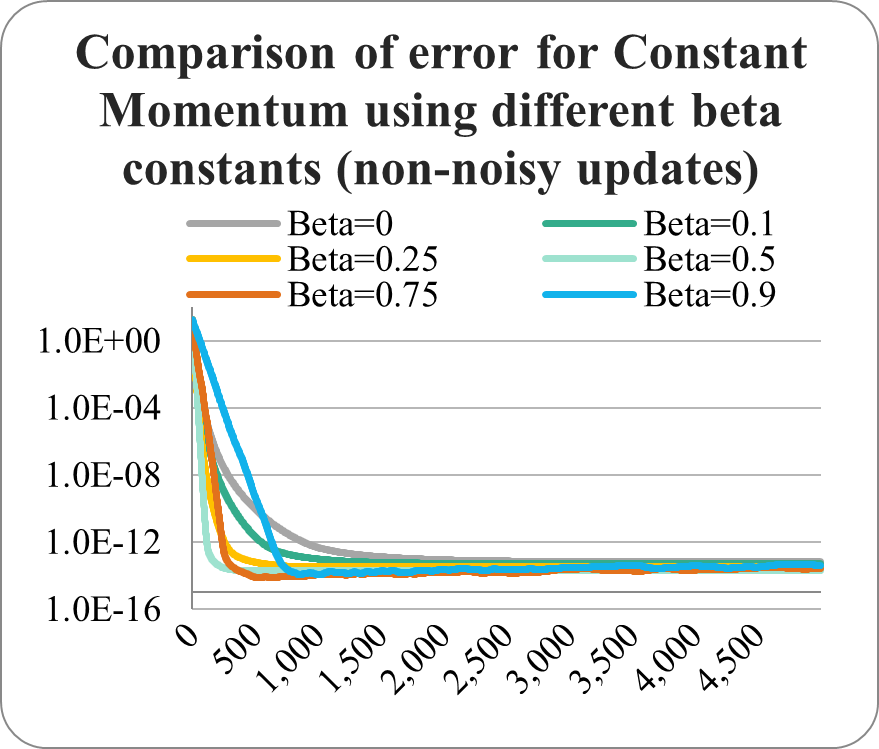}
			\caption{Non Stochastic Case}
			\label{fig:Beta Comparison Non Stochastic}
		\end{subfigure}%
		\hfill
		\begin{subfigure}{.5\linewidth}
			\includegraphics[scale=1]{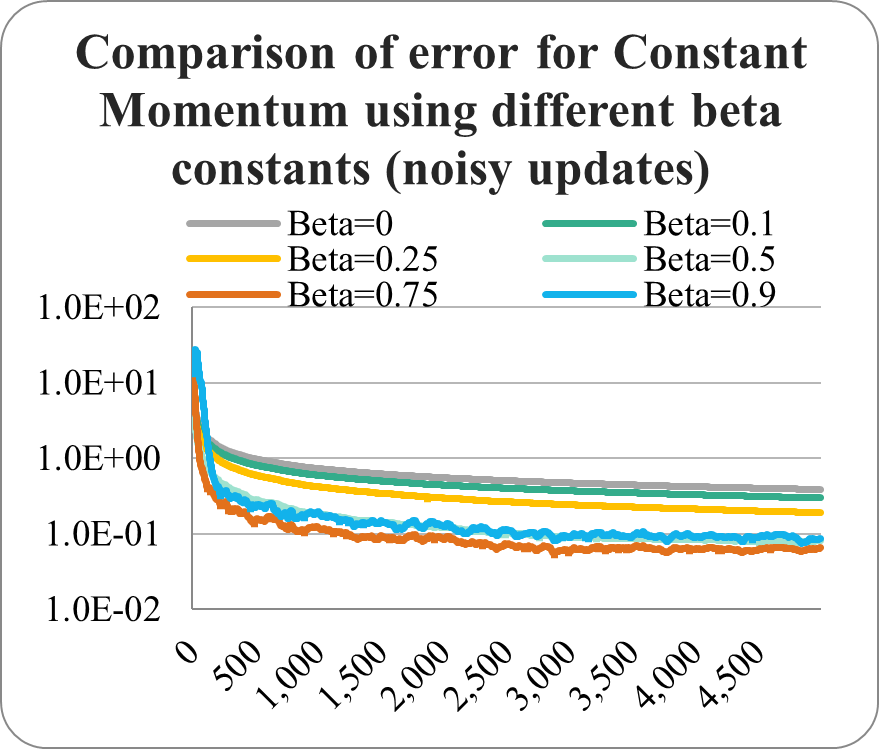}
			\caption{Stochastic Case}
			\label{fig:Beta Comparison Stochastic}
		\end{subfigure}
		\caption{Comparison of different Momentum}
		\label{fig: Beta Comparisons}
	\end{figure}
	
	We note that in the non-stochastic case, all values of $\beta \in [0,1)$ lead to convergence. Given enough iterations, we expect the same in the stochastic case as well.
	
	%\section{Showing Convergence with Momentum}
	\section{SHOWING CONVERGENCE WITH MOMENTUM FOR THE SCALE INVARIANT MONTE-CARLO (SIM) ALGORITHM}
	\label{appendixsection: Showing convergence with momentum}
	
	\subsection{Problem Setup}
	
	Our original stochastic approximation equation with momentum can be written as $$w_{k+1} = w_k- \alpha_k  \sltotau\left[\dfrac{\phi_i^\tr w - \widetilde{V}_i}{||\phi_i||^2}\phi_i\right] + \beta (w_k-w_{k-1})$$
	where the notations have the usual meaning explained in \hyperref[section: notation and preliminaries]{section \ref{section: notation and preliminaries}} and further, $\beta \in [0,1)$. We want to show that this converges, where we have already shown that the update $w_{k+1} = w_k- \alpha_k  \sltotau\left[\dfrac{\phi_i^\tr w - \widetilde{V}_i}{||\phi_i||^2}\phi_i\right]$ converges.
	
	\textbf{Approach used}
	
	Traditional algorithms may attempt such a momentum under the two timescale approximation scheme. These have been considered in \cite{Borkar,LakshmiNarayanBhatnagar}. Two time scale approximation are also considered in \cite{avrachenkov2020online} in the context of web page change rate estimation.
	We take a different approach. 
	First we convert the given stochastic approximation equation with momentum into a two timescale regime, with two iterates getting updated. 
	Then we collapse the second iterate into a perturbation on the first iterate $w_k$, and thus show convergence. 
	We detail this in the following sections.
	
	\subsection{Adapting the stochastic-approximation equation with momentum into a two timescale structure:}
	\begin{appendixpropn}
		The update equation $w_{k+1} = w_k- \alpha_k  \sltotau\left[\dfrac{\phi_i^\tr w - \widetilde{V}_i}{||\phi_i||^2}\phi_i\right] + \beta (w_k-w_{k-1})$ can also be written as the set of equations
		\begin{align*}
			w_{k+1} - w_k &= \alpha_k z_k \\
			z_0 &= TP_k(w_k)\\
			z_i &= z_{i-1} + \zeta_{(i,k)} TP_{k-i}(w_{k-i})\quad \forall i \in [1,k]		\end{align*}
		where $\zeta_{(i,k)}  = \beta^i \frac{\alpha_{k-i}}{\alpha_{k}}$
	\end{appendixpropn}
	\begin{proof}
		Consider:
		\begin{align}
			w_{k+1} &= w_k - \alpha_k  \sltotau\left[\dfrac{\phi_i^\tr w_k - \widetilde{V}_i}{||\phi_i||^2}\phi_i\right] + \beta (w_k-w_{k-1})
			\intertext{Rewriting as a difference:}
			w_{k+1} - w_k &= \alpha_k  \sltotau\left[\dfrac{\phi_i^\tr w_k - \widetilde{V}_i}{||\phi_i||^2}\phi_i\right] + \beta (w_k-w_{k-1})\nonumber
			\intertext{We will call the term $\sltotau\left[\dfrac{\phi_i^\tr w_k - \widetilde{V}_i}{||\phi_i||^2}\phi_i\right]$ as $TP_k(w_k)$}
			w_{k+1} - w_k &= \alpha_k  TP_k(w_k) + \beta (w_k-w_{k-1})\nonumber
			\intertext{Expanding the momentum term}
			w_{k+1} - w_k &= \alpha_k  TP_k(w_k)  + \beta (\alpha_{k-1}  TP_{k-1}(w_{k-1}) + \beta (w_{k-1}-w_{k-2}))\nonumber\\
			&= \alpha_k  TP_k(w_k)  + \beta \alpha_{k-1}  TP_{k-1}(w_{k-1}) + \beta^2 (w_{k-1}-w_{k-2})\nonumber\\
			&= \alpha_k  TP_k(w_k)  + \nonumber\\
			&\quad\qquad \beta \alpha_{k-1}  TP_{k-1}(w_{k-1}) + \beta^2 (\alpha_{k-2}  TP_{k-2}(w_{k-2}) + \beta (w_{k-2}-w_{k-3}))\nonumber\\
			&=\quad \dots\nonumber
			\intertext{Thus we can write the whole thing as:}
			&=\alpha_k TP_k(w_k) + \nonumber\\
			&\quad\qquad\beta \alpha_{k-1}  TP_{k-1}(w_{k-1}) + \beta^2 \alpha_{k-2}  TP_{k-2}(w_{k-2})+\dots+\beta^k\alpha_0TP_0(w_0)
		\end{align}
		We note that this is in the form of a discounted sum of vectors, which we have to bring into a form that is the sum of two iterates \cite{Kushner1997}.
		
		We reverse the order of the second iterate set. We build $z_k$ bottom up as follows. Let:
		\begin{align}
			z_0 &= TP_k(w_k)\nonumber\\
			z_1 &= z_0 + \beta \frac{\alpha_{k-1}}{\alpha_k} TP_{k-1}(w_{k-1})\nonumber\\
			\dots\enspace &= \quad \dots\nonumber\\
			z_{k} &= z_{k-1} + \beta^{k} \frac{\alpha_{0}}{\alpha_{k}} TP_{0}(w_{0})\nonumber
			\intertext{Further, to simplify this set of equations, we let $\zeta_{(i,k)}$ be the step size corresponding to $z_i$ such that $\zeta_{(i,k)} =\beta^i \frac{\alpha_{k-i}}{\alpha_{k}}$. Then we have the set of equations as:}
			w_{k+1} - w_k &= \alpha_k z_k\nonumber \\
			z_0 &= TP_k(w_k)\nonumber\\
			z_1 &= x_0 + \zeta_{(1,k)} TP_{k-1}(w_{k-1})\nonumber\\
			\dots\enspace &= \quad \dots\nonumber\\
			z_{k} &= z_{k-1} + \zeta_{(k,k)} TP_{0}(w_{0})\nonumber
			\intertext{Or more generally if $\zeta_{(i,k)} =\beta^i \frac{\alpha_{k-i}}{\alpha_{k}}$:}
			w_{k+1} - w_k &= \alpha_k z_k \\
			z_0 &= TP_k(w_k)\nonumber\\
			z_i &= z_{i-1} + \zeta_{(i,k)} TP_{k-i}(w_{k-i})\quad \forall i \in [1,k]\nonumber
		\end{align}
	\end{proof}
	
	\subsection{Collapsing the two iterate stochastic approximation equations into a single iterate form:}
	Now wish to express the above equation in terms of an expected update and a Martingale noise term (with respect to the filtration). For $z_0$, such an expression is straightforward: We add and subtract the expectation to change the equation from $z_0 = TP_k(w_k)$ to 
	\begin{equation}
		z_0 = \expectation \left[ TP_k(w_k)\big|\F_k\right] + (TP_k(w_k) - \expectation \left[ TP_k(w_k)\big|\F_k\right])
	\end{equation}
	where the first term is the expected update term $H_{(0,k)} = \expectation \left[ TP_k(w_k)\big|\F_k\right]$ second term is a martingale difference noise term, $M_{(0,k)} = (TP_k(w_k) - \expectation \left[ TP_k(w_k)\big|\F_k\right])$. 
	
	Now let us focus on $z_i$ for $i \in [1,k]$
	\begin{align}
		z_i &= z_{i-1} + \zeta_{(i,k)} TP_{k-i}(w_{k-i})\nonumber
		\intertext{can be rewritten as:}
		z_i &= z_{i-1} + \zeta_{(i,k)} \left[TP_{k-i}(w_{k}) + \left(TP_{k-i}(w_{k-i})-TP_{k-i}(w_{k})\right)\right]\nonumber
		\intertext{Which can be further broken down as:}
		z_i &= z_{i-1} + \zeta_{(i,k)} \bigg[\E [TP_{k-i}(w_{k})|\F_k]+\\
		&\qquad\qquad\qquad\left(TP_{k-i}(w_{k}) - \E \left[TP_{k-i}(w_{k})|\F_k\right]\right)+ \left(TP_{k-i}(w_{k-i})-TP_{k-i}(w_{k})\right)\bigg]\nonumber
		\intertext{Now we take an expectation of the first term over all possible $\F_k$. Thus the first term breaks into:}
		&= z_{i-1} + \zeta_{(i,k)} \bigg[\E_{\F_k}\left[\E [TP_{k-i}(w_{k})|\F_k]\big|\F_k\right] + \\
		&\quad\qquad\qquad\qquad\qquad(\E [TP_{k-i}(w_{k})|\F_k]-\E_{\F_k}\left[\E [TP_{k-i}(w_{k})|\F_k]\big|\F_k\right]) + \nonumber\\
		&\quad\qquad\qquad\qquad\qquad \left(TP_{k-i}(w_{k}) - \E \left[TP_{k-i}(w_{k})|\F_k\right]\right)+ \left(TP_{k-i}(w_{k-i})-TP_{k-i}(w_{k})\right)\bigg]\nonumber
		\intertext{Note that in the filtration, $\E_{\F_k}\left[\E [TP_{k-i}(w_{k})|\F_k]\big|\F_k\right]$ remains unaffected, and therefore, we can write $\E_{\F_k}\left[\E [TP_{k-i}(w_{k})|\F_k]\big|\F_k\right] =\E_{\F_k}\left[\E [TP_{k-i}(w_{k})|\F_k]\right]$. For ease of notation, we simply write $\E_{\F_k}\left[\E [TP_{k-i}(w_{k})|\F_k]\right]$ as $\E\left[\E [TP_{k-i}(w_{k})|\F_k]\right]$. Then we have:}
		&= z_{i-1} + \zeta_{(i,k)} \bigg[\E\left[\E [TP_{k-i}(w_{k})|\F_k]\right] + (\E [TP_{k-i}(w_{k})|\F_k]-\E\left[\E [TP_{k-i}(w_{k})|\F_k]\right])\nonumber\\
		&\qquad\qquad\qquad+ \left(TP_{k-i}(w_{k}) - \E \left[TP_{k-i}(w_{k})|\F_k\right]\right)+ \left(TP_{k-i}(w_{k-i})-TP_{k-i}(w_{k})\right)\bigg]\nonumber
	\end{align}
	Notice that the third term above $TP_{k-i}(w_{k}) - \E \left[TP_{k-i}(w_{k})|\F_k\right]$ is actually 0 as the filtration provides the exact hyperplanes as well as $w_k$. Thus the expression is deterministic. Therefore, $\E \left[TP_{k-i}(w_{k})|\F_k\right] = TP_{k-i}(w_{k})$. Thus we finally have
	\begin{align}
		z_i &= z_{i-1} + \zeta_{(i,k)} \bigg[\E\left[\E [TP_{k-i}(w_{k})|\F_k]\right] +\\ &\qquad\qquad\qquad\qquad
		(\E [TP_{k-i}(w_{k})|\F_k]-\E\left[\E [TP_{k-i}(w_{k})|\F_k]\right]) +\nonumber\\
		&\qquad\qquad\qquad\qquad \left(TP_{k-i}(w_{k-i})-TP_{k-i}(w_{k})\right)\bigg]\nonumber
	\end{align}
	
	Let 
	\begin{equation*}
		\begin{split}
			H_{(i,k)} &= \E\left[\E [TP_{k-i}(w_{k})|\F_k]\right]\\
			M_{(i,k)} &= \E [TP_{k-i}(w_{k})|\F_k]-\E\left[\E [TP_{k-i}(w_{k})|\F_k]\right]\\
			\varepsilon_{(i,k)} &= TP_{k-i}(w_{k-i})-TP_{k-i}(w_{k})
		\end{split}
	\end{equation*}
	
	If $\zeta_{(i,k)} =\beta^i \frac{\alpha_{k-i}}{\alpha_{k}}$. Further, let $h_{(i,k)}(\cdot)$ be some limiting o.d.e that asymptotically tracks $H_{(i,k)}(\cdot)$. Thus we have the set of equations:
	\begin{align}
		w_{k+1} - w_k &= \alpha_k z_k \\
		z_0 &= h_{(0,k)} + M_{(0,k)}\nonumber\\
		z_i &= z_{i-1} + \zeta_{(i,k)} (h_{(i,k)}+M_{(i,k)}+\varepsilon_{(i,k)})\quad \forall i \in [1,k]\nonumber
	\end{align}
	We collapse these now into a single equation. Since $h_{(i,k)} (w_k) = TP(w_k) \forall i,k$ based on \hyperref[appendixpropn: Expectation of Expectation is TP]{proposition \ref{appendixpropn: Expectation of Expectation is TP}} We will label this simply as $h(w_k)$
	
	Let $\h(w_k) = h(w_k)\left(1 + \sum\limits_{i=1}^k \zeta_{(i,k)}\right)$, $\Mdot_k = M_{(0,k)} + \sum\limits_{i=1}^k \zeta_{(i,k)} M_{(i,k)}$ and $\vepsilondot_k = \sum\limits_{i=1}^k \zeta_{(i,k)} \varepsilon_{(i,k)}$, Then:
	\begin{equation}
		w_{k+1} - w_k = \alpha_k[\h(w_k) + \vepsilondot_k + \Mdot_k]
	\end{equation}
	
	Now we have to show that this single equation follows the requirements for convergence. We will show each of the assumptions in order. 
	
	\subsection{Showing basic properties of required for convergence:}
	\label{appendixsubsection: A1 and A2 Lipschitz and step size}
	
	\begin{appendixpropn}
		\label{appendixpropn:single iterate with momentum has proper step size sequence}
		The step size sequence $\{\alpha_i\}_{i=1}^\infty$ satisfies $\sltoinf \alpha_i = \infty$ and $\sltoinf\alpha_i^2 < \infty$
	\end{appendixpropn}
	\begin{proof}
		The step size sequence remains the same as in \hyperref[appendixpropn:single iterate has proper step size sequence]{proposition \ref{appendixpropn:single iterate has proper step size sequence}}. Thus the proof remains the same.
	\end{proof}
	
	\begin{appendixpropn}
		\label{appendixpropn:mathring zeta is bounded}
		Let $\mathring{\zeta}_k := 1 + \sum\limits_{i=1}^k \zeta_{(i,k)}$. Then $\mathring{\zeta}_k$ is bounded. 
	\end{appendixpropn}
	\begin{proof}
		Consider 
		\begin{align*}
			\mathring{\zeta}_k &=\frac{1}{\alpha_k} \left[\alpha_k + \beta\alpha_{k-1} + \beta^2 \alpha_{k-2} + \dots\right]\\
			\intertext{Recall that $\alpha_k= \eta_k\vartheta_k$ where $\eta_k = \dfrac{1}{k^p};\enspace p\in (0.5,1]$ and $\vartheta_k = \dfrac{||TP_k(w_k)||}{||\Delta TP_k(w_k)||}$. Further, recall that $\sup \vartheta_k = \overline{\vartheta}$. Then:}
			\mathring{\zeta}_k&\leq \frac{\overline{\vartheta}}{\vartheta_k}\left[1+\beta \left(\frac{k}{k-1}\right)^p +\beta^2 \left(\frac{k}{k-2}\right)^p+\dots \right] \\
			&\leq \frac{\overline{\vartheta}}{\vartheta_k}\left[1+\beta \frac{k}{k-1} +\beta^2 \frac{k}{k-2}+\dots \right] \\
			&= \frac{\overline{\vartheta}}{\vartheta_k}[1+ \beta + \frac{\beta}{k-1} +\beta^2 + \frac{2\beta^2}{k-2}+\beta^3 + \frac{3\beta^3}{k-3}+\dots ] \\
			&= \frac{\overline{\vartheta}}{\vartheta_k}[(1+ \beta+\beta^2+\dots) + (\frac{\beta}{k-1} + \frac{2\beta^2}{k-2} + \frac{3\beta^3}{k-3}+\dots) ]
		\end{align*}

		As $k\to\infty$, the first half of the above expression is $\frac{\overline{\vartheta}}{\vartheta_k(1-\beta)}$. As $k\to\infty$, the second half converges to 0 \cite{KaviRamaMurthy}. Thus the entire expression remains bounded. 
		
	\end{proof}
	
	\begin{appendixpropn}
		\label{appendixpropn:single iterate is Lipschitz}
		The expected update for $w$, $\h:\R^n\to\R^n$ is Lipschitz
	\end{appendixpropn}
	\begin{proof}
		We have already shown that $h(\cdot) = TP(w)$ is Lipschitz (as can be seen from the fact that $\slsins\pi_s\dfrac{\phi_s\phi_s\tr}{||\phi_s||^2} w + C$ where $C = \slsins\pi_s\dfrac{\phi_sV_s}{||\phi_s||^2}$ is linear in w). Now we will show that $\h(\cdot)=h(\cdot)(1 + \sum\limits_{i=1}^k \zeta_i)$ is also Lipschitz. But we have shown that $\mathring{\zeta}_k = 1 + \sum\limits_{i=1}^k \zeta_{(i,k)}$ is bounded in \hyperref[appendixpropn:mathring zeta is bounded]{proposition \ref{appendixpropn:mathring zeta is bounded}}. 
		
		Thus we have that if $h(\cdot)$ is Lipschitz, then $\mathring{\zeta}_k h(\cdot) = \h(\cdot)$ is also Lipschitz for some constant $\mathring{\zeta}_k$.
		
	\end{proof}
	\subsection{Showing that the noise term is a martingale difference sequence:}
	\label{appendixsubsection: A3 Martingale difference noise}
	\begin{appendixpropn}
		\label{appendixpropn:M0k is bounded}
		We specifically consider $M_{(0,k)}$ first.
		$\E[M_{(0,k)}|\F_k]=0$ and $\E\left[||M_{(0,k)}||^2|\F_k\right]\leq K[1+||w_k||^2]$
	\end{appendixpropn}
	\begin{proof}
		We note that $M_{(0,k)} = TP_k(w_k) - \expectation \left[ TP_k(w_k) \big|\F_k\right]= TP_k(w_k) - TP(w_k)$  as per \hyperref[appendixSection: Update equation without momentum Standard Form]{appendix section \ref{appendixSection: Update equation without momentum Standard Form}}.
		
		Now we have already shown in \hyperref[appendixSection: Margingale difference sequence square integrable]{appendix section \ref{appendixSection: Margingale difference sequence square integrable}} that $\E[TP_k(w_k) - TP(w_k)|\F_k]=0$. 
		
		Further we also showed $\exists A_{(0,k)},C_{(0,k)}$ such that $M_{(0,k)} = A_{(0,k)}w_k - C_{(0,k)}$ whence $\E\left[||TP_k(w_k) - TP(w_k)||^2|\F_k\right]\leq K[1+||w_k||^2]$ for some $K\in\R$

	\end{proof}
	
	\begin{appendixpropn}
		\label{appendixpropn: Mik is bounded}
		$\E[M_{(i,k)}|\F_k]=0$ and $\E\left[||M_{(0,k)}||^2|\F_k\right]\leq K[1+||w_k||^2]$
	\end{appendixpropn}
	\begin{proof}
		First we note that 
		\begin{align*}
			M_{(i,k)} &= \E [TP_{k-i}(w_{k})|\F_k]-\E_{\F_k}\left[\E [TP_{k-i}(w_{k})|\F_k]\right]\\
			\intertext{Note that the second expectation remains unchanged given the filtration, thus we can rewrite this as:}
			&=\E [TP_{k-i}(w_{k})|\F_k]-\E_{\F_k}\left[\E [TP_{k-i}(w_{k})|\F_k]\big|\F_k\right]
		\end{align*}
		
		Given such a definition, $\E [M_{(i,k)}| \F_k] = \E_{\F_k}\left[\E [TP_{k-i}(w_{k})|\F_k]\big|\F_k\right]-\E_{\F_k}\left[\E [TP_{k-i}(w_{k})|\F_k]\big|\F_k\right] =0$
		
		For the second part, note that the filtration gives us the hyperplanes, say $\{1,\dots\T\}$ that have been sampled. Then:
		\begin{align*}
			\E [TP_{k-i}(w_{k})|\F_k] &= \left[\frac{1}{\T} \sum\limits_{i = 1}^{\T} \dfrac{\phi_i ^\tr w_k - \widetilde{V}_i}{||\phi_i||_2^2}\phi_i\right]\\
			\intertext{We obtain from \hyperref[appendixpropn: Expectation of Expectation is TP]{appendix proposition \ref{appendixpropn: Expectation of Expectation is TP}} that $\E\left[\E [TP_{k-i}(w_{k})|\F_k]\right] = TP(w_k)=\slsins\pi_s \left[\dfrac{\phi_s^\tr w - V_s}{||\phi_s||^2}\phi_s\right]$. Therefore:}
			M_{(i,k)} &=\left[\frac{1}{\T} \sum\limits_{i = 1}^{\T} \dfrac{\phi_i ^\tr w_k - \widetilde{V}_i}{||\phi_i||_2^2}\phi_i\right] -\slsins\pi_s \left[\dfrac{\phi_s^\tr w - V_s}{||\phi_s||^2}\phi_s\right]\\
			&= \left[\frac{1}{\T} \sum\limits_{i = 1}^{\T} \dfrac{\phi_i\phi_i ^\tr}{||\phi_i||_2^2} - -\slsins\pi_s \dfrac{\phi_s\phi_s^\tr }{||\phi_s||^2} \right]w_k-\left[\frac{1}{\T} \sum\limits_{i = 1}^{\T} \dfrac{\widetilde{V}_i\phi_i}{||\phi_i||_2^2} - -\slsins\pi_s \dfrac{V_s\phi_s }{||\phi_s||^2} \right]\\
			&= A_{(i,k)} w_k - C_{(i,k)}
		\end{align*}
		which is linear in $w_k$ with bounded coefficients. Further note that $A_{(i,k)}$ is bounded above as $\frac{\phi_s\phi_s^\tr}{||\phi_s||^2}$ has maximum eigen value 1. Further, $C_{(i,k)}$ is bounded as $\widetilde{V}$ is bounded above by $\dfrac{R_{max}}{1-\gamma}$ where $R_{max}$ is the maximum reward and $\gamma$ is the discounting factor.
		
		Thus $||M_{(i,k)}||^2$ is quadratic in $w_k$. Now it is easy to see that there would exist some K such that $\E\left[ ||M_{(i,k)}||^2 \big| \F_k\right] \leq K(1+||w_k||^2)$
		
	\end{proof}
	
	\begin{appendixpropn}
		\label{appendixpropn:single iterate has Martingale difference noise}
		Consider the filtration $\F_k = \{w_0,\dots,w_k\}$. Then the sequence $\{\Mdot_k\}$ is a zero-mean martingale difference noise sequence. Specifically, we have that:
		\begin{enumerate}
			\item $\E[\Mdot_k|\F_k] = 0$
			\item $\E[||\Mdot_k||^2\big|\F_k]\leq K_i(1+ ||w_k||^2)$
		\end{enumerate}
	\end{appendixpropn}
	\begin{proof}
		For the first part, we need to show $\E[\Mdot_k|\F_k] = 0$ where $\Mdot_k = M_{(0,k)} + \sum\limits_{i=1}^k \zeta_{(i,k)} M_{(i,k)}$. We have:
		\begin{align*}
			\E[\Mdot_k|\F_k] &= \E[ M_{(0,k)} + \sum\limits_{i=1}^k \zeta_{(i,k)} M_{(i,k)}|\F_k]
			\intertext{By linearity of expectation:}
			&= \E[ M_{(0,k)} |\F_k] +\sum\limits_{i=1}^k  \zeta_{(i,k)} \E[ M_{(i,k)}|\F_k]
			\intertext{But we have from \hyperref[appendixpropn: Mik is bounded]{proposition \ref{appendixpropn: Mik is bounded}} that $\E[ M_{(i,k)}|\F_k] = 0$ and from \hyperref[appendixpropn:M0k is bounded]{proposition \ref{appendixpropn:M0k is bounded}} that $\E[ M_{(0,k)}|\F_k] = 0$. Therefore}
			\E[\Mdot_k|\F_k] &= 0 +\sum\limits_{i=1}^k  \zeta_{(i,k)} 0\\
			&=0
		\end{align*}
		For the second part, we see this by linearity. 
		\begin{align*}
			\Mdot_k &= M_{(0,k)} + \sum\limits_{i=1}^k \zeta_{(i,k)} M_{(i,k)}
			\intertext{From propositions \hyperref[appendixpropn:M0k is bounded]{\ref{appendixpropn:M0k is bounded}} and \hyperref[appendixpropn: Mik is bounded]{\ref{appendixpropn: Mik is bounded}}, we can write the above as:}
			&= (A_{(0,k)} + \sum\limits_{i=1}^k \zeta_{(i,k)} A_{(i,k)})w_k - (C_{(0,k)} + \sum\limits_{i=1}^k \zeta_{(i,k)} C_{(i,k)})
			\intertext{Since $A_{(i,k)}, C_{(i,k)}$ are bounded $\forall i \in \{0,\dots,k\}$, we can write the above as:}
			\Mdot_k &= \mathring{A} w_k - \mathring{C}
		\end{align*}
		where $\mathring{A} = A_{(0,k)} + \sum\limits_{i=1}^k \zeta_{(i,k)} A_{(i,k)}$ is bounded and $\mathring{C} = C_{(0,k)} + \sum\limits_{i=1}^k \zeta_{(i,k)} C_{(i,k)}$ is bounded.
		Now we see that $\Mdot_k$ is linear in $w_k$ with bounded coefficients. 
		
		Thus $||\Mdot_k||^2$ is quadratic in $w_k$, whence 
		$\exists K\in \R$ such that $\E\left[||\Mdot_k||^2\big| \F_k\right]\leq K(1+||w_k||^2)$

	\end{proof}

	\subsection{Showing that the momentum terms sum to a perturbation:}
	\label{appendixsubsection: A5 Momentum as perturbation}
	\begin{appendixpropn}[Helper proposition for \ref{appendixpropn: varepsilon is perturbation}]
		\label{appendixpropn: zeta i (wk - wk-i) is bounded}
		$\sum\limits_{i=1}^k \zeta_i \cdot||w_k-w_{k-i}|| \to 0$ as $k\to \infty$
	\end{appendixpropn}
	\begin{proof}
		$\sum\limits_{i=1}^\infty \zeta_i \cdot||w_k-w_{k-i}|| = \sum\limits_{i=1}^m \zeta_i \cdot||w_k-w_{k-i}|| + \sum\limits_{i=m+1}^\infty \zeta_i \cdot||w_k-w_{k-i}||$
		
		Now given any $\epsilon$, there $\exists m$ such that $\sum\limits_{i=m+1}^\infty \zeta_i \cdot||w_k-w_{k-i}|| < \epsilon$. This is because $\zeta_i\sim \beta^i$ go to 0 and $||w_k-w_{k-i}||$ are bounded (shown separately when we show stability of iterates).
		
		Then given any finite m, at the asymptote as $k\to \infty$, we have $||w_k-w_{k-i}||\to 0$ as $\alpha_k \to 0$. Thus $\sum\limits_{i=1}^k \zeta_i \cdot||w_k-w_{k-i}|| <\epsilon$ as $k\to \infty$ for any arbitrary $\epsilon$. 
		
		Thus $\sum\limits_{i=1}^k \zeta_i \cdot||w_k-w_{k-i}|| \downarrow 0$ as $k\to \infty$
	\end{proof}
	
	\begin{appendixpropn}
		\label{appendixpropn: varepsilon is perturbation}
		$\vepsilondot_k$ are perturbation terms that satisfy 
		$||\vepsilondot_k|| \leq d_k (1+||w_k||)$ where ${d_k}$ are a sequence of positive scalars such that $\lim\limits_{k\to\infty}d_k = 0$
	\end{appendixpropn}

	\begin{proof}
		First note that $\vepsilondot_k = \sum\limits_{i=1}^k \zeta_i \varepsilon_{(i,k)}$ and $\varepsilon_{(i,k)} = TP_{k-i}(w_{k-i})-TP_{k-i}(w_{k})$. Thus 
		\begin{align}
			\varepsilon_i &= \frac{1}{\T} \sum\limits_{j=1}^\T \left[\dfrac{\phi_j^\tr w_{k-i} - V_j}{||\phi_j||^2}\right]\phi_j - \frac{1}{\T} \sum\limits_{j=1}^\T \left[\dfrac{\phi_j^\tr w_{k} - V_j}{||\phi_j||^2}\right]\phi_j\nonumber\\
			\intertext{Taking terms common:}
			&= \frac{1}{\T} \left(\sum\limits_{j=1}^\T \left[\dfrac{\phi_j\phi_j^\tr }{||\phi_j||^2}\right]\right)(w_{k-i}-w_k)\nonumber\\
			||\varepsilon_i|| &\leq ||w_{k-i}-w_k||\label{appendixeqn:varepsiloni}
		\end{align}
		Now we extend this by using $\vepsilondot_k = \sum\limits_{i=1}^k \zeta_i \varepsilon_i$
		\begin{align}
			\vepsilondot_k &= \sum\limits_{i=1}^k \zeta_i \varepsilon_i\nonumber
			\intertext{Expanding $\varepsilon_i$ using the inequality in \eqref{appendixeqn:varepsiloni}}
			||\vepsilondot_k||&\leq \sum\limits_{i=1}^k \zeta_i \cdot||w_k-w_{k-i}||\nonumber\\
			\intertext{Now we note that asymptotically as $k\to \infty$, we have for finite i, $||w_k-w_{k-i}|| \to 0$ and for large i, $\zeta_i\to 0$. Thus by \hyperref[appendixpropn: zeta i (wk - wk-i) is bounded]{proposition \ref{appendixpropn: zeta i (wk - wk-i) is bounded}} the above is bounded above by some arbitrary $\epsilon$.}
			&\leq \epsilon\nonumber
		\end{align}
		Thus asymptotically we see that this perturbation term is o(1).
	\end{proof}

	\subsection{Stability Criterion: Iterates remain bounded}
	\label{appendixsubsection: A4 Stability Criterion}
	In \cite{Borkar}, we have to prove that the iterates of in the update equation remain bounded. \cite{LakshmiNarayanBhatnagar} have provided a stability criterion to ensure that the iterates remain bounded. While we have already shown that the iterates on after the expected update remain bounded in \hyperref[propn:MonteCarloA4]{proposition \ref{propn:MonteCarloA4}}, here we will explicitly show the stability criterion is satisfied. 
	
	But first a basic proposition:
	\begin{appendixpropn}
		\label{appendixpropn: Expectation of Expectation is TP}
		$\E\left[\E [TP_{k-i}(w_{k})|\F_k]\right] = TP(w_k)$
	\end{appendixpropn}
	\begin{proof}
		Note that we are considering the expectation over all possible filtrations. Note that the random variables under consideration are $\T$ - the number of hyperplanes sampled, $i\in\{1,\dots,\T\}$ - the set of hyperplanes sampled, and $\widetilde{V}$ - the value function.
		The filtration gives us $w_k$ and the set of hyperplanes chosen in a particular trajectory. Let's label the unique hyperplanes in the trajectory by $\{1,\dots\T\}$. Then:
		\begin{align*}
			\E_{\F_k}\left[\E [TP_{k-i}(w_{k})|\F_k]\right] &= \E_{\F_k=(\T,i,\widetilde{V})} \left(\frac{1}{\T}\sum\limits_{i=1}^{\T} \left[ \dfrac{\phi_i^\tr w_k-\widetilde{V}_i}{||\phi_i||^2}\phi_i\right]\right)
			\intertext{By linearity we rewrite this as:}
			&=  \E_{\T}\left(\frac{1}{\T}\sum\limits_{i=1}^{\T} \E_{(i,\widetilde{V})} \left[ \dfrac{\phi_i^\tr w_k-\widetilde{V}_i}{||\phi_i||^2}\phi_i\right]\right)
			\intertext{Over all possible filtrations, we can write the expectation of the inner term as:}
			\E_{(i,\widetilde{V})} \left[ \dfrac{\phi_i^\tr w_k-\widetilde{V}_i}{||\phi_i||^2}\phi_i\right] &= \slsins \pi_s \dfrac{\phi_s^\tr w_k - V_s}{||\phi_s||^2}\phi_s
			\intertext{Substituting this in the previous expression, we get:}
			\E_{\F_k}\left[\E [TP_{k-i}(w_{k})|\F_k]\right] &=\E_{\T}\left(\frac{1}{\T}\sum\limits_{i=1}^{\T}\left[\slsins \pi_s \dfrac{\phi_s^\tr w_k - V_s}{||\phi_s||^2}\phi_s\right] \right)
			\intertext{But the inner expression is now independent of $\T$. Thus:}
			\E_{\F_k}\left[\E [TP_{k-i}(w_{k})|\F_k]\right] &=\left[\slsins \pi_s \dfrac{\phi_s^\tr w_k - V_s}{||\phi_s||^2}\phi_s\right]
		\end{align*}
		We note that the RHS is $TP(w_k)$
		
	\end{proof}
	\begin{appendixpropn}
		Let us define the sequence of functions $\h_c(w):\R^n\mapsto\R^n$ such that $\h_c(w) = \dfrac{\h(cw)}{c};\quad c\geq 1$. Then 
		\begin{enumerate}
			\item $\h_c(\cdot) \mapsto \h_\infty(\cdot)$ as $c\to\infty$ uniformly on compact sets
			Further,
			\item The limiting o.d.e, $\dot{w}(t) = \h_{\infty}(w(t))$ has a unique globally asymptotically stable equilibrium at the origin.
		\end{enumerate}
	\end{appendixpropn}
	
	\begin{proof}
		First note that 
		\begin{align*}
			\h(w) &= h(w)(1 + \sum\limits_{i=1}^k \zeta_i)\\
			\intertext{But $\mathring{\zeta} = (1 + \sum\limits_{i=1}^k \zeta_i)$. Then:}
			&=\mathring{\zeta} h(w)
			\intertext{Expanding $h(w)$:}
			&=\mathring{\zeta} \slsins\pi_s\left[\dfrac{\phi_s^\tr w - V_s}{||\phi_s||^2}\right]\phi_s
			\intertext{Now we write $\h_c(w)$ from its definition:}
			\h_c(w) &= \mathring{\zeta} \slsins\pi_s\left[\dfrac{\phi_s^\tr cw - V_s}{c||\phi_s||^2}\right]\phi_s
			\intertext{But the constants c can be cancelled for the $w$ term:}
			&= \mathring{\zeta} \left(\slsins\pi_s\left[\dfrac{\phi_s\phi_s^\tr}{||\phi_s||^2}\right] w - \frac{1}{c} \slsins\pi_s\left[\dfrac{V_s\phi_s}{||\phi_s||^2}\right]\right)
		\end{align*}
		We observe the uniform convergence of this set of functions $\h_c(w)$ to $\h_\infty(w)$ in the limit $c\to\infty$ as the term $c$ is only involved with a constant coefficient given by $\slsins\pi_s\left[\dfrac{V_s\phi_s}{||\phi_s||^2}\right]$. Thus the first part is proved.
		
		For the second part, we note the following:
		\begin{align*}
			\h_\infty(w) &= \lim\limits_{c\to\infty} \mathring{\zeta} \left(\slsins\pi_s\left[\dfrac{\phi_s\phi_s^\tr}{||\phi_s||^2}\right] w - \frac{1}{c} \slsins\pi_s\left[\dfrac{V_s\phi_s}{||\phi_s||^2}\right]\right)
			\intertext{Now we apply the limit only on the second term:}
			&=  \mathring{\zeta} \slsins\pi_s\left[\dfrac{\phi_s\phi_s^\tr}{||\phi_s||^2}\right] w  - \left(\lim\limits_{c\to\infty} \frac{\mathring{\zeta}}{c}\right) \slsins\pi_s\left[\dfrac{V_s\phi_s}{||\phi_s||^2}\right]
			\intertext{Evaluating the limit, we get 0 for the second term:}
			&=  \mathring{\zeta} \slsins\pi_s\left[\dfrac{\phi_s\phi_s^\tr}{||\phi_s||^2}\right] w  - 0\\
			\intertext{Thus:}
			\h_\infty(w) &=  \mathring{\zeta} \slsins\pi_s\left[\dfrac{\phi_s\phi_s^\tr}{||\phi_s||^2}\right] w
		\end{align*}
		Now consider the system $\dot{w}(t) = \h_{\infty}(w(t)) = \mathring{\zeta} \slsins\pi_s\left[\dfrac{\phi_s\phi_s^\tr}{||\phi_s||^2}\right] w$. At the equilibrium point, 
		\begin{align*}
			\dot{w}(t) &= 0\\
			\mathring{\zeta} \slsins\pi_s\left[\dfrac{\phi_s\phi_s^\tr}{||\phi_s||^2}\right] w &= 0
			\intertext{But $\Phi$ has full column rank (by assumption). Thus no eigen value of $\slsins\pi_s\left[\dfrac{\phi_s\phi_s^\tr}{||\phi_s||^2}\right]$ is 0. Thus:}
			w&=0
		\end{align*}
		Thus $\dot{w}(t) = \h_{\infty}(w(t))$ has a unique globally asymptotically stable equilibrium at the origin.
	\end{proof}
	\subsection{The stochastic update equation with momentum converges:}
	
	In \hyperref[appendixsubsection: A1 and A2 Lipschitz and step size]{section \ref{appendixsubsection: A1 and A2 Lipschitz and step size}}, we showed the assumptions A1 and A2 required for convergence. In \hyperref[appendixsubsection: A3 Martingale difference noise]{section \ref{appendixsubsection: A3 Martingale difference noise}} we showed that the noise term is a martingale difference sequence - assumption A3 (per \cite{Borkar}). In \hyperref[appendixsubsection: A4 Stability Criterion]{section \ref{appendixsubsection: A4 Stability Criterion}}, we showed assumption A4, which was the stability criterion required to show that the iterates remain bounded \cite{LakshmiNarayanBhatnagar}. Finally, in \hyperref[appendixsubsection: A5 Momentum as perturbation]{section \ref{appendixsubsection: A5 Momentum as perturbation}}, we showed that the momentum terms added a perturbation term to the o.d.e that we are asymptotically tracking.
	
	All that is left to see is where we converge to.
	\begin{appendixpropn}
		The globally asymptotically stable equilibrium for the limiting o.d.e $\dot{w}(t) = \h(w(t))$ that our stochastic approximation equation tracks is given by $w^* = \left[\left(\Phi^\tr N D N \Phi\right)^{-1} \Phi^\tr N D N\right] V$
	\end{appendixpropn}
	\begin{proof}
		The update equation, $\dot{w}(t) = \h(w(t))$ can be written as $\dot{w}(t) = \mathring{\zeta} h(w(t)) = \mathring{\zeta} \slsins\pi_s\left[\dfrac{\phi_s^\tr w - V_s}{||\phi_s||^2}\right]\phi_s$. Considering that the states are sampled from the stationary distribution $\pi$, we have:
		\begin{align}
			\dot{w}(t) = \mathring{\zeta} \slsins \pi_s\left[\dfrac{\phi_s^\tr w - V_s}{||\phi_s||^2}\right]\phi_s\nonumber
			\intertext{Then the equilibrium point is given by the point where:}
			\dot{w}(t) &= 0\nonumber\\
			\mathring{\zeta} \slsins\pi_s \left[\dfrac{\phi_s^\tr w - V_s}{||\phi_s||^2}\right]\phi_s &= 0\nonumber
			\intertext{But $\mathring{\zeta}$ is just a constant. Therefore:}
			\slsins\pi_s \left[\dfrac{\phi_s^\tr w - V_s}{||\phi_s||^2}\right]\phi_s &= 0
		\end{align}
		But this is an equation that we have already solved in \hyperref[appendixsubsection:Convergence point of the Scale Invariant Monte Carlo]{section \ref{appendixsubsection:Convergence point of the Scale Invariant Monte Carlo}}. The solution is given by $w^* = \left[\left(\Phi^\tr N D N \Phi\right)^{-1} \Phi^\tr N D N\right] V$
	\end{proof}
	
	We've now satisfied all the criteria and also shown the point to which we converge. Thus we show the convergence for the full algorithm with momentum.
	
\end{document}